\documentclass[letterpaper]{article}
\usepackage{aaai23}  % DO NOT CHANGE THIS
\usepackage{times}  % DO NOT CHANGE THIS
\usepackage{helvet}  % DO NOT CHANGE THIS
\usepackage{courier}  % DO NOT CHANGE THIS
\usepackage[hyphens]{url}  % DO NOT CHANGE THIS
\usepackage{graphicx} % DO NOT CHANGE THIS
\usepackage{natbib}  % DO NOT CHANGE THIS AND DO NOT ADD ANY OPTIONS TO IT
\usepackage{caption} % DO NOT CHANGE THIS AND DO NOT ADD ANY OPTIONS TO IT
\usepackage{algorithm}
\usepackage{algorithmic}

\usepackage{cite}
\usepackage{amsmath,amssymb,amsfonts,amsthm}
\usepackage{mathrsfs}

\usepackage{listings}
\DeclareCaptionStyle{ruled}{labelfont=normalfont,labelsep=colon,strut=off} % DO NOT CHANGE THIS
\floatstyle{ruled}
\newfloat{listing}{tb}{lst}{}
\floatname{listing}{Listing}

\usepackage{graphicx}
\usepackage{textcomp}
\usepackage{xcolor}

\usepackage{comment}
\usepackage{xparse}
\usepackage{mathtools}
\usepackage{subcaption}
\usepackage{enumerate}

\usepackage{xparse}
\usepackage{mathtools}
\usepackage{tikz}
\usepackage{subcaption}
\usepackage{enumerate}

\newtheorem{theorem}{Theorem}[]
\newtheorem{example}{Example}[]

\newtheorem{lemma}[theorem]{Lemma}
\newtheorem{problem}[theorem]{Problem}

\usepackage{paralist}
\usepackage{todonotes}
\DeclareDocumentCommand{\states}{D<>{} O{}  t'}{\mathsf{S}_{#1}^{\IfBooleanTF{#3}{\prime~#2}{#2}}}
\DeclareDocumentCommand{\state}{D<>{} O{}  t'}{\mathsf{s}_{#1}^{\IfBooleanTF{#3}{\prime #2}{#2}}}
\DeclareDocumentCommand{\action}{D<>{} O{} t'}{\mathsf{a}_{#1}^{\IfBooleanTF{#3}{\prime#2}{#2}}}
\DeclareDocumentCommand{\Av}{D<>{} O{} t' D(){}}{\mathsf{Av}_{#1}^{\IfBooleanTF{#3}{\prime}{}#2}\ifthenelse{\isempty{#4}}{}{(#4)}}
\newcommand{\trans}{\mathbb{T}}
\newcommand{\transR}{\mathsf{R}}
\newcommand{\initstate}{s_0}

\newcommand{\av}{\mathsf{Av}}

\NewDocumentCommand{\actions}{d()}{{\IfNoValueTF{#1}{\mathsf{Act}}{\fun{\mathsf{Act}}{#1}}}}

\newcommand{\Reals}{\mathbb{R}}
\newcommand{\rat}{\mathbb{Q}}
\newcommand{\nat}{\mathbb{N}}
\newcommand{\realpos}{\mathbb{R}_{\geq 0}}
\newcommand{\ratpos}{\mathbb{Q}_{\geq 0}}
\newcommand{\Distributions}{\mathcal{D}}
\newcommand{\dist}{\ensuremath{p} }
\newcommand{\supp}{\ensuremath{\mathsf{Supp}} }
\newcommand{\Mdp}{\mathcal{M}}
\newcommand{\atomicprop}{\textbf{AP}}
\newcommand{\labelling}{L}
\newcommand{\embeddedMdp}{\mathcal{M}_{\mathcal{E}}}
\newcommand{\transP}{P}
\newcommand{\irun}{Runs^{\Mdp}}
\newcommand{\frun}{FRuns^{\Mdp}}

\newcommand{\oautomata}{\mathcal{A}}
\newcommand{\alphabets}{\Sigma}
\newcommand{\ostate}{Q}
\newcommand{\oinitstate}{q_0}
\newcommand{\otrans}{\delta}
\newcommand{\acceptingc}{F}
\newcommand{\oruns}{\mathcal{R}_{\oautomata}}

\newcommand{\ctmc}{\Mdp^{[\sigma]}}

\newcommand{\objective}{\mathcal{O}}
\newcommand{\discobjective}{\textsf{DR}}
\newcommand{\avgobjective}{\textsf{AR}}

\newcommand{\unif}{\Mdp_{C}}

\newcommand{\pmdp}{\Mdp \times \oautomata}
\newcommand{\augmdp}{\Mdp^{\zeta}}

\newcommand{\valueq}{\mathcal{Q}_{\sigma}}

\newcommand{\kvaluef}{\Qf^{(k)}}
\newcommand{\newvaluef}{\Qf^{(k+1)}}

\newcommand{\avgreward}{Ar^{\pmdp^{[\sigma]}}}

\newcommand{\dfactor}{\alpha}
\newcommand{\ctmdprate}{\gamma_{\alpha}}
\newcommand{\valuesigma}{v^{\sigma}}
\newcommand{\pmtrx}{P_{\Mdp}}
\newcommand{\drate}{\beta}
\newcommand{\bschedule}{\sigma^*}
\newcommand{\thrate}{\beta_{0}}
\newcommand{\valuestar}{v^{\sigma^*}}
\newcommand{\uMdp}{M_C}
\newcommand{\upmtrx}{P_{\uMdp}}
\newcommand{\udrate}{\ctmdprate^{C}}
\newcommand{\uthrate}{\gamma_{\alpha_0}^{C}}
\newcommand{\ctmdpthrate}{\gamma_{\alpha_{0}}}
\newcommand{\avgrew}{g}
\newcommand{\emdp}{\mathcal{M}_\mathcal{E}}

\newcommand{\PSemSat}{{\sf PSem}}
\newcommand{\PSat}{{\sf PSat}}
\newcommand{\ESemSat}{{\sf ESem}}
\newcommand{\ESat}{{\sf ESat}}

\newcommand{\Ll}{\mathcal{L}}

\newcommand{\set}[1]{\left\{ #1 \right\}}

\newcommand{\bstrategy}{\sigma^*}
\newcommand{\Qf}{\mathcal{Q}_f}

\usetikzlibrary{automata,positioning,decorations.markings,arrows,intersections,%    
calc,shapes,fit}                                     
\colorlet{darkgreen}{green!40!black}                                                                   \colorlet{darkblue}{blue!60!black}                                                                     \colorlet{darkred}{red!50!black}                                                                       \colorlet{safecellcolor}{yellow!5}                                                                     \colorlet{goodcellcolor}{green!10}                                                                     \colorlet{badcellcolor}{blue!10}      

\tikzset{
    >=latex,node distance=2cm,on grid,auto, initial text=, 
    box state/.style={draw,rectangle,minimum size=8mm,rounded corners},
    prob state/.style={draw,very thick,shape=circle,darkblue,minimum size=3mm,inner sep=0mm},
    every loop/.style={shorten >=0pt},
    accepting state/.style={double distance=1.2pt, outer sep = 0.6pt+\pgflinewidth},
    accepting dot/.style={above=-2.7pt,circle,fill,darkgreen,inner sep=2pt,radius=1pt}, 
    loop above/.append style={every loop/.append style={out=120, in=60, looseness=6}},
    loop below/.append style={every loop/.append style={out=300, in=240, looseness=6}},
    loop left/.append style={every loop/.append style={out=210, in=150, looseness=6}},
    loop right/.append style={every loop/.append style={out=30, in=330, looseness=6}} 
}  

\DeclareCaptionStyle{ruled}{labelfont=normalfont,labelsep=colon,strut=off} % DO NOT CHANGE THIS
\usepackage{algorithm}
\usepackage{algorithmic}
\usepackage{temporal}
\usepackage{newfloat}
\usepackage{listings}
\floatstyle{ruled}
\newfloat{listing}{tb}{lst}{}
\floatname{listing}{Listing}

\title{Reinforcement Learning for Omega-Regular Specifications \\on Continuous-Time MDP}
\author{
    Amin Falah\textsuperscript{\rm 1},
    Shibashis Guha\textsuperscript{\rm 2},
    Ashutosh Trivedi\textsuperscript{\rm 1}
}
\affiliations{
    \textsuperscript{\rm 1} University of Colorado Boulder\\
    \textsuperscript{\rm 2} Tata Institute of Fundamental Research
}

\begin{document}
    \maketitle

\begin{abstract}
    Continuous-time Markov decision processes (CTMDPs) are canonical models to express sequential decision-making under dense-time and stochastic environments. 
    When the stochastic evolution of the environment is only available via sampling, model-free reinforcement learning (RL) is the algorithm-of-choice to compute optimal decision sequence. 
    RL, on the other hand, requires the learning objective to be encoded as scalar reward signals.
    Since doing such translations manually is both tedious and error-prone, a number of techniques have been proposed to translate high-level objectives (expressed in logic or automata formalism) to scalar rewards for discrete-time Markov decision processes (MDPs). Unfortunately, no automatic translation exists for CTMDPs.
    
    We consider CTMDP environments against the learning objectives expressed as omega-regular languages. 
    Omega-regular languages generalize regular languages to infinite-horizon specifications and can express properties given in popular linear-time logic LTL.
    To accommodate the dense-time nature of CTMDPs, we consider two
    different semantics of omega-regular objectives: 
    1) \emph{satisfaction semantics} where the goal of the learner is to
    maximize the probability of spending positive time in the good states, and
    2) \emph{expectation semantics} where the goal of the learner is to optimize
    the long-run expected average time spent in the ``good states" of the automaton.
    We present an approach enabling correct translation to scalar reward signals
    that can be readily used by off-the-shelf RL algorithms for CTMDPs. 
    We demonstrate the effectiveness of the proposed algorithms by evaluating it on some popular CTMDP benchmarks with
    omega-regular objectives. 
\end{abstract}
% \maketitle

\section{Introduction}
\label{sec:intro}
Reinforcement learning (RL)~\cite{Sutton18} is a sequential optimization approach where a decision maker learns to optimally resolve a sequence of choices based on feedback received from the environment. This feedback often takes the form of rewards and punishments with strength proportional to the fitness of the decisions taken by the agent as judged by the environment towards some higher-level learning objectives.
\emph{This paper develops convergent RL algorithms for continuous-time Markov decision processes (CTMDP) against learning requirements expressed in $\omega$-regular languages~\cite{Baier08}.}

% The rest of the introduction provides a motivation for the aforementioned problem while comparing it to relevant related works. 

\paragraph{Need for Reward Translation.} 
Due to a combination of factors---including the success of deep neural networks~\cite{Goodfe16}
and a heavy intellectual and monetary investment from the industry and the
academe~\cite{Mnih15,Silver16,Levine16}---RL has emerged as a leading human-AI collaborative design paradigm
where the key role of the human designers reduces to designing the appropriate
scalar reward signals, while the RL algorithm creates an optimal schedule driven by the reward signal.  
Unfortunately, then, the de-facto communication between the human designers and the RL algorithms is quite rigid: it forces the human programmers to think in the language suitable for the learning agents and not in a way that comes naturally to humans: declarative or imperative languages. 
To meet this challenge, a recent
trend is to enable logic~\cite{sadigh2014learning,Bozkurt0ZP20,camacho2019ltl,li2017reinforcement}  and
automatic
structures~\cite{HahnPSSTW19,icarte2018using,IcarteKlassenValenzanoMcIlraith20}
to express learning intent in RL.
The common thread among these approaches is to encode the specification as an automaton based reward structure and derive scalar rewards with every
discrete interaction with the environment.
However, when the problem domain is continuous-time, aforementioned approaches are not
applicable as they support the discrete-time semantics modeled as finite-state Markov decision processes (MDP or DTMDP for emphasis). 

This paper aims to enable the use of RL in unknown CTMDPs against high-level specifications expressed as $\omega$-automata~\cite{VW86,Baier08}.

\vspace{-0.5em}
\paragraph{Continuous-Time Reinforcement Learning.} 
Semi-MDPs~\cite{Baykal2011} model environments where the interaction between the decision maker and the environment may occur at any dense time point.
CTMDPs~\cite{guo2009continuous} are subclasses of semi-Markov decision processes where
% the learner chooses an action and 
the exact time and the resolution of the next
state is governed by an exponential distribution with a \emph{rate} parameter that is dependent on the current state and the action chosen.
The classical RL algorithms for DTMDPs
have been elegantly generalized to CTMDPs for both discounted~\cite{BD94} and average~\cite{das1999solving} objectives.
We employ the Q-learning algorithm for CTMDP~\cite{BD94} to compute optimal
schedules for $\omega$-regular learning objectives.

\paragraph{The $\omega$-Regular Objectives.}
Finite automata on infinite words---or $\omega$-automata---may be equipped with a variety of equally expressive infinitary acceptance conditions (e.g., deterministic Rabin and nondeterministic B\"uchi) with well-understood succinctness and complexity trade-offs. 
From their first application in solving Church's synthesis problem~\cite{thomas2009facets} to becoming the \emph{lingua franca} in expressing specifications of safety-critical systems~\cite{Baier08}, $\omega$-automata are a key part of the computational backbone to automated verification and synthesis.
Linear temporal logic (LTL)~\cite{Baier08} is a popular declarative language to express properties of infinite sequences.
Specifications expressed using $\omega$-automata form a strict superset of specifications expressed as LTL formulas. Given an LTL formula, one can effectively construct an $\omega$-automaton~\cite{VW86}.
For this reason, we focus on $\omega$-automata based specifications.

The expanding role of RL in safety-critical systems has prompted the use of $\omega$-automata in expressing learning objectives due to improved expressiveness and interpretability over scalar rewards.
In this work, we use nondeterministic B\"uchi automata to express $\omega$-regular specifications.

\paragraph{Continuous-Time in B\"uchi Automata.}
B\"uchi automata are finitary structures accepting infinite sequences of letters 
that visit a distinguished set of good (accepting) states infinitely often.  
For scheduling problems over stochastic systems modeled as DTMDPs, the optimal schedules can be specified via schedules that maximize the measure of accepted system
behaviors. 
While for discrete-time system the naturalness of such discrete infinitary
visitation semantics is well-established, for continuous-time systems it is
imperative that the acceptance criterion must heed to the actual time spent in
such good states.  
Two distinct interpretations of good dense-time behavior are natural:
While the focus of the \emph{satisfaction semantics} is on maximizing the
measure of behaviors that visit good states infinitely often, the \emph{expectation semantics} focuses on maximizing the long-run expected time
spent in good states.
% While \emph{expectation semantics} focuses on maximizing the long-run expected time spent in good states, the focus of the \emph{satisfaction semantics} is on maximizing the measure of behaviors that visit good states infinitely often.
We develop RL algorithms for CTMDPs with B\"uchi specifications
under both semantics.

A recent work~\cite{oura2022learning} studies an alternative objective for semi-MDPs against multi-objective specifications composed of an $\omega$-regular objectives (satisfaction semantics) and a risk objective (expected risk). 
The key distinction between \citeauthor{oura2022learning}'s approach and ours (vis-\`a-vis the satisfaction semantics) is that the former is based on bounded synthesis paradigm that requires a bound parameter on co-B\"uchi states visitation and thus reduces the specification to a safety objective (where reward translation is straightforward). In contrast, our approach does not require any bound from the practitioner and is capable of handling general $\omega$-regular objectives.
Moreover, the expectation semantics has not been explored in any existing literature. 

% \paragraph{Related Work.}
% In this section, we compare our approach on satisfaction semantics with \cite{oura2022learning}. 
% CTMDP ... 

% Continuous time stochastic logic (CSL)~\cite{baier1999CSL} is a stochastic branching time temporal logic and is used to model-check CTMCs against properties expressed in it.
% The logic contains a steady state operator $\mathcal{S}$ to express limiting distribution properties, and our expectation semantics can be captured using the $\mathcal{S}$ operator. However, in our work, we consider a CTMDP environment, and we are are interested in the \emph{synthesis} of a schedule rather than the model-checking problem.

\paragraph{Contributions.}
% The contributions of the paper are as follows.
Our key contributions are as follows:
\begin{enumerate}
    \item 
    We present a novel (expectation) semantics for B\"uchi automata to capture time-critical properties for CTMDPs. 
    \item 
    We present procedures to translate B\"uchi automata with satisfaction and expectation semantics to reward machines \cite{icarte2018using} in a form that enables application of the off-the-shelf CTMDP RL algorithms.
    We show that one needs distinct reward mechanisms for these two semantics, and we
    establish the correctness and effectiveness of these reward translations.
    \item As a by-product of the proofs, we provide  a simplified proof of existence of Blackwell optimal schedules~\cite{puterman2014markov} in CTMDPs based on uniformization. 
    \item 
    We present an experimental evaluation to demonstrate the effectiveness of the proposed approach.
\end{enumerate}

Due to space constraints, the detailed proofs and other omitted information (such as details of the benchmarks) are provided as part of the supplementary material.

 \section{Satisfaction Vs. Expectation Semantics}
\label{sat vs expt}

% As a motivation for the expectation semantics, consider the situation where a good state signifies some desirable behavior (time spent in performing some critical aspect of the mission), then the satisfaction semantics
% require that such states are visited infinitely often, whereas the expectation semantics provide incentives to policies that maximize the expected
% time spent in such states.  
% Conversely, in some situations one may argue against the expectation semantics being inadequate as
% they incentivize expected behavior instead of good behavior of a greater measure of runs.  
% Depending upon the context, one of these semantics may be more apt than the other.
% Let us consider a mockup to explicate the utility of these semantics.
The following example gives an intuition on satisfaction and expectation semantics.
\begin{example}[Satisfaction or Expectation?] \label{example:1}
The CTMDP shown in Figure~\ref{fig:grid-world} (adapted from \cite{HahnPSSTW19}) represents four zones (represented by $s_0$ to $s_3$ in the figure) on the Mars surface. 
Suppose that a mission to Mars arrives in Zone~$0$ (a known, safe territory) and is expected to explore the terrain in a safe fashion, gather and transmit information, and stay alive to maximize the return on the mission.
For simplicity, assume that Zone $1$ (purple) models a crevasse harmful to the safe operations, while zones $2$ and $3$ are central to exploration mission and are analogous in their information contents.

\begin{figure*}[t]
\noindent\resizebox{0.3\textwidth}{!}{
 \begin{minipage}{0.33\textwidth}
\begin{tikzpicture}[shorten >=1pt]
\begin{scope}
\node (l0) [state,ellipse,initial above, fill = safecellcolor]   {$s_0$};
\node (l1) [state, ellipse, fill = badcellcolor] at (-2.1,0)   {$s_1$};
\node (l2) [state, ellipse, fill = goodcellcolor] at (2.1,0)   {$s_2$};
\node (l3) [state,ellipse, fill = goodcellcolor] at (0,-2.1)   {$s_3$};
% \end{scope}
% \begin{scope}[every node/.style={scale=0.8}]

% \draw[rounded corners,->, thick] (l3) -- (-4.5,-4.3) -- node [below] {$b,(\lambda(0,b){-}r)$}  (l5);
\path [->, thick,above=.5cm,align=center]
    (l0) edge node [above] {$b$} node [below, font = \scriptsize]{$\lambda(0,b)-r$}   (l1)
    (l0) edge [bend right]  node [below] {$b,r$}   (l2)
    (l2) edge [bend right]  node [above] {$f$}   (l0)
    (l0) edge [bend right]  node [left] {$a$}   (l3)
    (l3) edge [bend right]  node [right] {$c$}   (l0)
    (l1) edge [loop above]  node [above] {$d$} ()
    (l2) edge [loop above]  node [above] {$e$} ()
    ;
\end{scope}
\end{tikzpicture}     
  \end{minipage}
  }
  \noindent\resizebox{0.3\textwidth}{!}{
  \begin{minipage}{0.33\textwidth}
   \begin{tikzpicture}
    \begin{scope}
    \node (l2) [state] {$q_2$};
    \node (l0) [state,initial] at (-2,2)   {$q_0$};
    \node (l1) [state,accepting] at (2,2)   {$q_1$};
    
\draw[rounded corners,->, thick] (l0) -- (-1, 3) -- node [above] {$(g \land \neg p)$} (1, 3) -- (l1);
\draw[->, thick] (l1) --  node [above] {$(\neg g \land \neg p)$} (l0);
\draw[rounded corners,->, thick] (l0) --  node [left] {$p$} (-2, 0) -- (l2);
\draw[rounded corners,->, thick] (l1) --  node [left] {$p$} (2, 0) -- (l2);
\path [->,thick]
    (l1) edge [loop above] node [above] {$(g \land \neg p)$} () 
    (l2) edge [loop below] node [below] {$\top$} ()
    (l0) edge [loop above] node [above] {$(\neg g \land \neg p)$}   ();
\end{scope}
\end{tikzpicture}
 \end{minipage}
 }
 \noindent\resizebox{0.3\textwidth}{!}{
\begin{minipage} {0.33\textwidth}
\centering
\begin{tikzpicture}[shorten >=1pt]
\begin{scope}
\node (l0) [state,ellipse,initial above]   {$s_0,q_0$};
\node (l1) [state, ellipse] at (-2,-1.5)   {$s_2,q_0$};
\node (l2) [state, ellipse] at (-4.5,0)   {$s_3,q_0$};
\node (l3) [state,ellipse,accepting] at (-4.5,-3.5)   {$s_0,q_1$};
\node (l4) [state,ellipse,accepting] at (-2,-3.5)   {$s_2,q_1$};
\node (l5) [state, ellipse] at (0.1,-4.3)   {$s_1,q_0$};
\node (l6) [state,ellipse] at (2,-4.3)   {$*,q_2$};
\draw[rounded corners,->, thick] (l3) -- (-4.5,-4.3) -- node [below] {$b,(\lambda(0,b){-}r)$}  (l5);
\path [->, thick,above=.5cm,align=center]
    (l0) edge node [above, sloped] {$b,r$}   (l1)
    (l0) edge  node [above] {$a$}   (l2)
    (l0) edge node [above,sloped] {$b,(\lambda(0,b){-}r)$}   (l5)
    (l2) edge [bend right] node [left] {$c$} (l3)
    (l3) edge  node [right] {$a$} (l2)
    (l3) edge node [above, sloped] {$b,r$} (l1)
    (l1) edge [bend right] node [above] {$f$} (l3)
    (l1) edge  node [left] {$e$} (l4)
    (l4) edge node [above] {$f$} (l3)
    % (l3) edge node [below] {$b,(\lambda(0,b){-}r)$} (l5)
    (l5) edge  node [above] {$d$} (l6)
    (l6) edge [loop above]  node [above] {$\top$} ()
    (l4) edge [loop right] node [right] {$e$} ()
    ;
\end{scope}
\end{tikzpicture}

  \end{minipage}
  }
% }
  \caption{The mars surveillance example where a CTMDP (left) can be in four different states where states $s_0$ has the label $\{\neg \mathtt{p} , \neg \mathtt{g}\}$, state $s_2$ and $s_3$ have label $\{\neg \mathtt{p}, \mathtt{g} \}$, and state $s_1$ have the label $\{\mathtt{p} ,\neg \mathtt{g}\}$. 
The rates of each transition (if not $1$) is written in the figure. 
  A deterministic B{\"u}chi automata for the $\omega$-regular objective $\varphi = (\always \neg \mathtt{p}) \wedge (\always\eventually \mathtt{g})$ (center). Product CTMDP (right) where each zone has two components for denoting the CTMDP and the B{\"u}chi automaton parts. All the zones whose second component is $q_2$ is combined as one. 
%   The accepting zones are coloured in green.
  The exit rate of an action from a zone $(s,q_i)$ in the product CTMDP is same as the exit rate of the action from $s$ in the original CTMDP. 
%   For example, the rate of action $b$ from zone $(0,q_0)$ is $\lambda(0,b)$.
}
 \label{fig:grid-world}
\end{figure*}
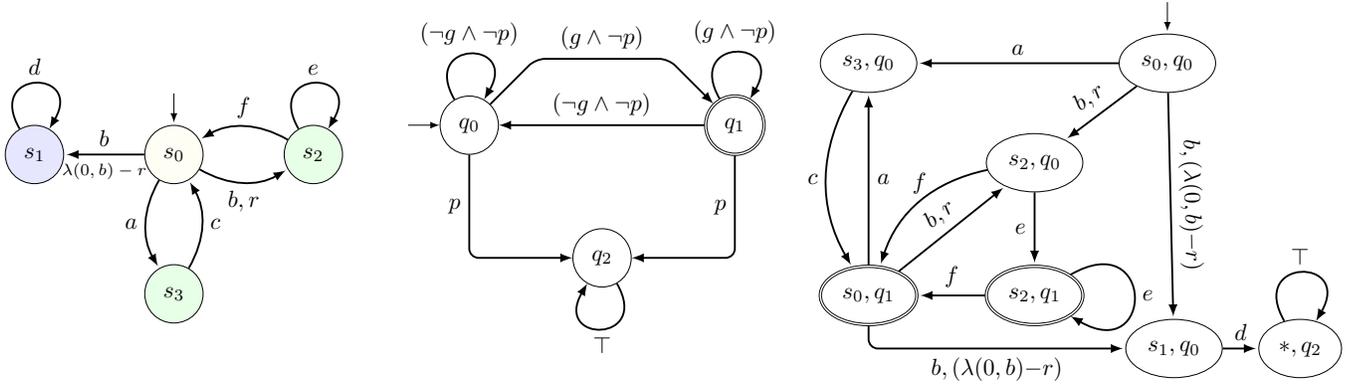

Given the unknown uncertainty of the terrain of Mars, the system is modeled as a CTMDP with associated uncertainty on the time of various actions where the exit rate of action $a$ from Zone (state) $0$ is denoted by $\lambda(0,a)$. In other words, when selected an action $a$ in a state $s$, the probability of spending $t$ time units in $s$ before taking $a$ is given by the cumulative distribution function $1 - e^{- \lambda(0,a)t}$.
Each transition have a rate associated with it and it determines the probability of taking that transition.
Assume that the action $b$ from Zone $0$ goes to Zone $2$ with rate $r$ (high probability) and to Zone $1$ with rate  $(\lambda(0,b) - r)$ (low probability).
% \textcolor{red}{there is an action $b$ and a label for state $b$, this would create confusion.}
The mission objective is to avoid Zone $1$ (purple zone) while infinitely often visiting the Zone $2$ or $3$ (the green zones). 
It can be captured in LTL~\cite{Baier08} as:
\[
\varphi = (\always \neg \mathtt{p}) \wedge (\always(\eventually \mathtt{g}))
\]
specifying that across the infinite horizon always (i.e. at every step expressed as temporal modality, $\always$) avoid the purple region ($\neg b$), and always eventually (i.e. at some time in the future expressed as temporal modality, $\eventually$) reach the green region, i.e. $(\always (\eventually g))$. 
The $GF \phi$ modality is often referred as \emph{infinitely often $\phi$}.
LTL combines these temporal operators using the standard propositional logic connectives such as: and ($\wedge$), or ($\lor$), not ($\neg$), and implication ($\to$).

This declarative specification can also be expressed using the B\"uchi automaton shown in Figure~\ref{fig:grid-world} (center) where the double circled states (here, $q_1$) denote accepting states.
The B\"uchi automaton can be used as a monitor to check the behavior of the learner over the environment.
For our example, it is visualized by taking the synchronous product (an extended space CTMDP) of the CTMDP with the automaton shown in Figure~\ref{fig:grid-world} (right).

For the satisfaction semantics on the product CTMDP, our goal is to maximize the probability that every infinite horizon behavior visits the accepting state infinitely often, while for the expectation semantics the goal is to maximize the expected time the system dwells in the accepting state.
\begin{itemize}
    \item \noindent {\bf Satisfaction Objective.} 
Consider the case where we have one Mars rover in this mission. Hence, our goal naturally is to maximize the probability of visiting green zones infinitely often while avoiding the purple zone (the satisfaction semantics).
In this case, the optimal schedule is to choose actions $a$ and $c$ indefinitely, i.e. the schedule $(a \to c)^\omega$, that satisfies the objective with probability $1$.
Note that it does not make sense to choose action $b$ no matter how low the probability is to reach the purple zone.
\item 
\noindent {\bf Expectation Objective.} Consider an alternative setting where we have a fleet of drones (we are okay in losing some drones as long as we maximize the mission objective) that needs to be sent to the surveillance of zone $2$ or $3$. 
Suppose that due to unforeseeable circumstances the mission may cease operation any time, and hence the goal is to maximize total expected time spent in the green zones ($2$ and $3$).
The schedule $(a\to c)^\omega$ is not optimal anymore as it may dwell a considerable amount in the Zone $0$. 
On the other hand any drone that chooses $b$ in Zone $0$ risks moving to Zone $1$ with a small probability.
As our goal is to maximize the expected time spent in the green zone over a large group of drones, the expectation semantics captures this intent and the optimal schedule is to start with action $b$.
\end{itemize}
% It can be observed that for the two objectives defined above, the optimal schedule is pure only in the product CTMDP, i.e, the optimal schedule for both expectation and satisfaction objectives require memory which is provided by the B\"uchi automaton. 
\end{example}

% \vspace{0.5em}
% \noindent\textbf{Organization.} 
% The paper is organized as follows. 
% We recall the technical preliminaries in Section~\ref{sec:prelims}, and define the two semantics for $\omega$-regular objectives and the key problem in Section~\ref{sec:p_statement}.
% In Section~\ref{sec:blackwell_optimality}, we give a simple uniformization based proof of existence of Blackwell optimal schedules in CTMDPs.
% The correctness of the reward translation for the expectation and satisfaction semantics is derived in Section~\ref{sec:d_time} and Section~\ref{sec:theorems&algo}.
% We demonstrate the effectiveness of our algorithms via some experimental results in Section~\ref{sec:expt}.

\section{Preliminaries}
\label{prelims}
\label{sec:prelims}
We write $\mathbb{N}, \rat$ and $\ratpos$ for the sets of natural numbers, rational numbers, and non-negative rational numbers, respectively.
For a natural number $n \in \nat$, we denote by $[n]$ the set $\set{1, \ldots, n}$.
% , and by $[n]_0$ the set $\{0, \dots, n\}$.
% \todo{Daddy cool: are these sets both supposed to be the same?}
Given a finite set $A$, a (rational) \textit{probability distribution} over $A$ is a function
$\dist \colon A \rightarrow [0, 1] \cap \rat$ such that $\sum_{a\in A} \dist(a) = 1$. 
% We call $A$ the \emph{domain of $p$}, and denote it by $\dom(p)$.
We denote the set of probability distributions on $A$ by $\Distributions(A)$. The \textit{support} of the
probability distribution $\dist$ on $A$ is $\supp(\dist) = \left\lbrace a {\in} A \;\vert\; \dist(a) {>} 0\right\rbrace$. A distribution is called \emph{Dirac} if $|\supp(\dist)| = 1$.

% An alphabet $\Sigma$ is a finite set of letters. 
% A finite word over $\Sigma$ is a finite concatenation of letters from $\Sigma$, while an $\omega$-word is an infinite such concatenation.
% We write $\Sigma^*$ and $\Sigma^\omega$ for the set of finite and infinite words over $\Sigma$.

\paragraph{Continuous-Time MDPs.}
% A \emph{Markov chain} is a tuple $M = (S,E,\transP)$, where $S$ is a set of states, $E \subseteq S \times S$ is a set of edges (we assume in this paper that the set $E(s)$ of outgoing edges from $s$ is nonempty and finite for all $s \in S$), and $\transP: S \to \Distributions(E)$ assigns a probability distribution---on the set $E(s)$ of outgoing edges from $s$---to all states $s \in S$. In the following, $\transP(s,(s,s'))$ is denoted $\transP(s,s')$, for all $s \in S$. A Markov chain $M$ is \emph{finite} if $S$ is finite.
% {\color{red} Do we need to define Markov chains?}
A (discrete-time) \emph{Markov decision process (MDP)} is a tuple of the form $\Mdp = (\states, \initstate, \actions, \trans)$, where $\states$ is a finite set of \emph{states}, $\initstate {\in} \states$ is the \emph{initial} state, $\actions$ is a finite set of \emph{actions}, and $\trans: \states {\times} \actions \to \Distributions(\states)$ is a \emph{transition function}. 
Let $\actions{}(s)$ be the set of actions enabled in the state $s {\in} S$.
An MDP is called a Markov chain if for every $s {\in} S$, the set $\actions{}(s)$ is singleton.
A \emph{continuous-time} MDP (CTMDP) is a tuple of the form $\Mdp = (\states, \initstate, \actions, \transR)$, where $\transR: \states {\times} \actions {\times} \states \to \Reals_{\geq 0}$ is a \emph{transition rate function}, while the rest of the parameters are as in the case of an MDP. 
% A \emph{continuous time Markov decision process (CTMDP)} is a tuple of the form $\Mdp = (\states, \initstate, \actions, \av, \transR, \atomicprop, \labelling)$, where $\transR: \states \times \actions \times \states \to \Reals_{\geq 0}$ is a \emph{transition rate matrix}, while the rest of the parameters are as in the case of an MDP.
For $s {\in} \states$ and $a {\in} \actions{}(s)$, we define $\lambda(s, a) = \sum_{s'} \transR(s,a,s') > 0$ to be the \emph{exit rate} of $a$ in $s$.
We define a \emph{probability matrix}, $P_{\mathcal{M}}$, where 
\[
    \transP_{\mathcal{M}}(s, a, s') = \begin{cases} 
    \frac{\transR(s, a, s')}{\lambda(s, a)} & \text{if } \lambda(s, a)>0 \\
    0 & \text{otherwise}
    \end{cases}
\]
 When $\mathcal{M}$ is clear from the context, we simply denote $P_{\mathcal{M}}$ by $P$.
The residence time for action $a$ in $s$ is exponentially distributed with mean $\lambda(s, a)$. For a given state $s$ and an action $a$, the  probability of spending $t$ time units in $s$ before taking the action is given by the cumulative distribution function $F(t|s,a) = 1 - e^{- \lambda(s,a)t}$
of the exponential distribution. The probability of a transition from state $s$ to $s'$ on an action $a$ in $t$ time units is given by 
$p_{a}(s,s',t) = P(s,a,s')\cdot F(t|s,a)$.
A CTMDP is called a continuous-time Markov chain (CTMC) if for every state $s \in S$, the set $\actions{}(s)$ is singleton.
% Let $\post(s,a) = \{t \:|\: \transR(s, a, t) > 0\}$, that is, the set of states that can be reached from $s$ through action $a$.
% An \emph{end-component} (EC) $M = (T,\actions{})$, with $\emptyset \neq T \subseteq \states$ and $\actions:T \rightarrow 2^{\actions{}}$ of an MDP $\Mdp$ is a \emph{sub-MDP} of $\Mdp$ such that: for all $s \in T$, we have that $\actions{}(s)$ is a subset of the actions available from $s$; for all $a \in \actions{}(s)$, we have $\post(s,a) \subseteq T$; and, it's underlying graph is strongly connected.
% An end-component is accepting if it contains an accepting state, otherwise, it is rejecting.

\paragraph{Uniformization.}
A \emph{uniform} CTMDP has a constant exit rate $C$ for all state-action pairs i.e, $\lambda(s,a) = C$ for all states $s \in \states$ and actions $a \in \actions{}(s)$. The procedure of converting a non-uniform CTMDP into a uniform one is known as \emph{uniformization}. Consider a non-uniform CTMDP $\Mdp$. Let $C {\in} \realpos$ be such that $C {\geq} \lambda(s,a)$ for all $(s, a) \in \states {\times} \actions$. We obtain a uniform CTMDP $\unif$ by changing the rates to $\transR'$:
    $$
    \transR'(s,a,s') = \begin{cases} 
    \transR(s,a,s') & \text{if } s \neq s' \\
    \transR(s,a,s')+C - \lambda(s,a) & \text{if } s=s' 
    \end{cases}
    $$
For every action $a \in \actions{}(s)$ from each state $s$ in the new CTMDP we have a self loop  if $\lambda(s,a) < C$. A uniformized CTMDP has a constant transition rate $C$ for all actions and because of this, the mean interval time between any two successive actions is constant.
See Appendix~\ref{uniform} for an example.

\paragraph{Schedules.}
An infinite run of the CTMDP is an $\omega$-word
$
(s_1, (t_1, a_1), s_2, (t_2, a_2), \ldots) \in \states \times ((\realpos \times \actions) \times \states)^\omega
$
where $s_i \in \states$, $a_i \in \actions{}(s_{i})$ and $t_i$ is the time spent on state $s_i$. A finite run is of the form $(s_1, t_1, a_1, \ldots, t_{n-1}, a_{n-1}, s_n)$ for some $n \in \mathbb{N}$. The set of infinite and the set of finite runs in $\Mdp$ are denoted by $\irun$ and $\frun$ respectively. Similarly, the set of infinite runs and the set of finite runs starting from a state $s$ in $\Mdp$ are denoted by $\irun(s)$ and $\frun(s)$ respectively. For $r \in \frun$, we denote by $last(r)$ the last state in the run $r$.

We use a \emph{schedule} to resolve non-determinism in a CTMDP. A schedule is a function $\sigma : \frun \rightarrow \Distributions(\actions)$, where $\Distributions(\actions)$ is a probability distribution on the set of enabled actions. Given a finite run $r \in \frun$, a schedule gives a probability distribution over all actions enabled in $last(r)$.
%  , i.e., if $p(a) \in [0,1]$ gives the probability of choosing action $a$ in $\sigma(r)$, then $\sum_{a \in \av(last(r))} p(a) = 1$.
 A schedule is \emph{deterministic} if $\Distributions(\actions)$ is Dirac, i.e, a single action is chosen in the distribution, otherwise it is \emph{randomized}. 
%  As CTMDPs deal with continuous time , a schedule also depends on the time spent in each state of the run r. A \emph{stationary} schedule is a schedule which depends only on $last(r)$. 
 Further, a schedule $\sigma$ is stationary if for all $r$,$r' \in \frun$ with $last(r) = last(r')$, we have that $\sigma(r) = \sigma(r')$. A \emph{pure} schedule is a deterministic stationary schedule. Let $\Sigma_{\Mdp}$ be the set of all schedules.
 
 A CTMDP $\Mdp$ under a schedule $\sigma$ acts as a continuous time Markov chain (CTMC) which is denoted by $\ctmc$.
 The set of infinite and the set of finite runs in $\ctmc$ are denoted by $\irun_\sigma$ and $\frun_\sigma$ respectively.
 The behavior of a CTMDP $\Mdp$ under a schedule $\sigma$ and starting state $s \in S$ is defined on a probability space
$(\irun_{\sigma}(s), \frun_{\sigma}(s), {\Pr}^{\Mdp}_{\sigma}(s))$ over
the set of infinite runs of $\sigma$ with starting state $s$.  Given a random variable 
% a real-valued random variable over the set of infinite runs
$f : \irun_{\sigma} \to \Reals$, we denote by $\mathbb{E}^{\Mdp}_{\sigma}(s) \{f\}$ the
expectation of $f$ over the runs of $\ctmc$.
For $n \geq 1$, we write $X_n$, $Y_n$, $D_n$, and $T_n$ for the random variables corresponding to the $n$-th state, action, time-delay in the $n$-th state, and time-stamp (time spent up to the $n$-th state). We let $D_0 = T_0 = 0$.
% These concepts for MDPs are defined analogously.
%  In the sequel, we use the terms `schedule' and `policy' interchangeably.
 
%  A {\it sub-MDP} of $\Mdp$ is an MDP $\Mdp' = (S', s_0'A', T')$, where $S' \subset
% S$, $A' \subseteq A$ is such that $A'(s) \subseteq A(s)$ for every $s \in S'$,
% and $T'$ is analogous to $T$ when restricted to $S'$ and
% $A'$. Moreover $\Mdp'$ is closed under probabilistic transitions.
% An {\it end-component}\cite{alma991027942769706011} of an MDP $\Mdp$ is a sub-MDP $\Mdp'$ such that for every state pair $s, s' \in S'$ there is a 
% schedule that can reach $s'$ from $s$ with positive probability. 
% A maximal end-component is an end-component that is maximal under set-inclusion.
% Every state $s$ of an MDP $\Mdp$ belongs to at most one maximal end-component.
% We can define sub-CTMDP, and end components in a CTMDP, as in the case of MDPs.
% Further, we can see that Theorem~\ref{thm:ec} also holds for CTMDPs. 

%  \begin{theorem}[End-Component Properties \cite{alma991027942769706011}]\label{thm:ec}
% Once an end-component $C$ of an MDP is entered, there is a schedule that visits every state-action pair in $C$ with probability 1 and stays in $C$ forever. Moreover, for every schedule the probability that a run ends up in an end-component is 1.  
% \end{theorem}
% We can define sub-CTMDP, and end components in a CTMDP, as in the case of MDPs.
% Further, we can see that Theorem~\ref{thm:ec} also holds for CTMDPs. 

\paragraph{Rewardful CTMDPs.}
A rewardful CTMDP ($\Mdp, rew$) is a CTMDP and a reward function $rew : \states \cup (\states \times \actions) \rightarrow \realpos$ which assigns a  \emph{reward-rate} to each state and a scalar reward to each state-action pair. Thus spending $t$ time-units in $s\in\states$ gives $rew(s) \cdot t$ of (state-delay) reward and choosing $a$ from $s$ gives $rew(s, a)$ (action) reward.

Continuous time discounting is done with respect to a discount parameter $\alpha {>} 0$ where one unit of reward obtained at time $t$ in the future gets a value of $e^{-\alpha t}$. Formally, the expected discounted reward for an arbitrary schedule $\sigma$ from a state $s$ is given by:
% Given a continuous discount parameter $\alpha {>} 0$, the expected discounted reward for an arbitrary schedule $\sigma$ from $s\in\states$, equals:
\begin{multline*}
\discobjective^{\Mdp[\sigma]} (\alpha)(s) =     \mathbb{E}_{\sigma}^{\Mdp}(s) \bigg[ \sum_{n=1}^{\infty} e^{-\alpha T_{n-1}}\Big( rew(X_n, Y_n) + \\ \int_{T_{n-1}}^{T_{n}}e^{-\alpha(t - T_{n-1})} rew(X_n)  dt\Big) \bigg].
 \end{multline*}
%  \begin{multline*}
% \discobjective^{\ctmc} (\alpha)(s) =     \mathbb{E}_{\sigma}^{\Mdp}(s) \bigg[ \sum_{n=1}^{\infty} e^{-\alpha T_{n-1}}\Big( rew(X_n, Y_n) + \\ \int_{T_{n-1}}^{T_{n}}e^{-\alpha(t - T_{n-1})} 
%  rew(X_n)  dt\Big) \bigg].
%  \end{multline*}
 Here, we multiply the expected reward obtained at the $n-$th state with $e^{- \alpha T_{n-1}}$ as per the continuous time discounting. The initial term in the parenthesis corresponds to the reward obtained from state $X_n$ by picking action $Y_n$ (action reward) and the second term corresponds to the state-delay reward i.e reward obtained with respect to the reward-rate $rew(X_n)$ which is discounted over the time $(t - T_{n-1})$. 
% On the other hand, the one-step expected average reward from state $s$ obtained by a state-action pair $(s,a)$ is given by $\rho(s,a) = rew(s,a) + \frac{rew(s)}{\lambda(s,a)}$. 

The expected average reward from $s$ under $\sigma$ is given by: 
 \begin{multline*}
 \avgobjective^{\Mdp[\sigma]}(s) = \liminf_{N\rightarrow \infty}  \mathbb{E}_{\sigma}^{\Mdp}(s) \bigg[ 
 \frac{1}{T_{N}} \cdot \Big(\sum_{n=1}^{N} rew(X_n, Y_n) +\\ \int_{T_{n-1}}^{T_{n}}
 rew(X_n) dt \Big) \bigg],
 \end{multline*}
where the first and second term corresponds to the action and state-delay reward respectively. 
Recall that $T_{N}$ is the total time spent upto the $n$-th state.
% \begin{equation*}
%     \begin{split}
%  \avgobjective^{\ctmc}(s) = \liminf_{n \rightarrow \infty}  \frac{\mathbb{E}_{\sigma}^{\Mdp}(s) \{\sum^{n}_{i=0} \rho(s_i,a_i)\}}{\mathbb{E}_{\sigma}^{\Mdp}(s) \{\sum_{i=0}^{n} \tau_i\}} \end{split}
% \end{equation*}
%  where $s_i$, $a_i$, and $\tau_i$ are the random variables for the $i$-th state, action, and dwell-time. 
Consider an objective $\objective \in 
 \{\discobjective, \avgobjective \}$. 
%  \{\reachobjective, \discobjective, \avgobjective \}$. 
 The expected reward obtained by schedule $\sigma$ on $s \in \states$ is denoted by $\objective^{\Mdp[\sigma]}\!(s)$. A schedule $\sigma^{*}$ is optimal for $\objective$ if 
    $\objective^{\Mdp[\sigma^*]}(s) = sup_{\sigma \in \Sigma_{\Mdp} }\objective^{\Mdp[\sigma]}\!(s)$ for all $s \in \states$.

% \subsection{Linear Temporal Logic (LTL)}
% Linear temporal logic(LTL) \cite{pnueli1977temporal} can describe a subset of $\omega$-regular objectives. The syntax of the logic contains the operators $\neg, \lor, \land, \F, \always, \X, \until$. Other operators can be derived from $\textbf{true},\neg, \lor, \X$ and $\until$ as:
% \begin{itemize}
%     \item $\textbf{false} \myeq \neg \textbf{true}$;
%     \item $\psi_1 \land \psi_2 \myeq \neg(\neg \psi_1 \lor \neg \psi_2)$;
%     \item $\F \psi \myeq \textbf{true } \until \text{ } \psi$;
%     \item $\always \psi \myeq \neg \F \neg \psi$.
% \end{itemize}     
% where $\psi,\psi_1,\psi_2$ are LTL formulas. An LTL property can be translated into \textbf{NRW}, \textbf{LDBW} and \textbf{SLDBW}. If an $\omega$-regular word $w$ satisfies an LTL formula $\psi$, we denote it by writing $w \models \psi$. Semantics of LTL is defined inductively \cite{Baier08}.

For a given CTMDP $\Mdp$, one can compute the optimal schedule for the discounted-sum objective or the expected average by using policy iteration, value iteration or linear programming \cite{feinberg-MDP, puterman2014markov} on the uniformized CTMDP $\unif$. 
% For $\omega$-regular objectives, when the CTMDP is completely known, we can find an accepting end-component, and get an optimal schedule that maximises the probability of reaching the said component. These methods however cannot be executed if we do not have information on the transition/reward structure of the CTMDP, and on the rates corresponding to the state-action pairs.
When the CTMDP is unknown (unknown rates and states), an optimal schedule can be computed via reinforcement learning.

\paragraph{Reinforcement Learning (RL).}
 RL allows us to obtain an optimal schedule by repeatedly interacting with the environment and thereby observing a reward. A \emph{training episode} is a finite sequence of states, actions and rewards which terminates on certain specified conditions like when the number of samples drawn is greater than some threshold. The RL obtains information about rates and rewards of the CTMDP model by running several training episodes. Broadly, there are two categories of RL, model-based and model-free. We focus on space efficient model-free RL algorithms as they compute optimal schedule without constructing the state transition system  \cite{strehl2006pac}.
%  Model-free RL do not explicitly estimate the transition rates, transition probabilities and rewards. 
 
 One of the most successful model-free learning algorithm for DTMDPs is the Q-learning algorithm~\cite{wd92}. It aims at learning (near) optimal schedules in a (partially unknown) MDP for the discounted sum objective.  Bradtke and Duff~\cite{BD94} introduced the Q-learning algorithm for CTMDPs. We give here a brief description of Q-learning algorithm for CTMDPs.

For a given discount parameter $\alpha {>} 0$, the one-step expected discounted reward for an action $a$ from state $s$ is given by 
$\rho(s,a) = rew(s,a) {+} \frac{rew(s)}{\alpha+\lambda(s,a)}$~\cite[Eq. 11. 5. 3]{puterman2014markov}.
The Q-function for a state $s$ and an action $a$ under schedule $\sigma$, denoted $\valueq(s,a)$, is defined as
\[
 \rho(s,a) + \frac{\lambda(s,a)}{\lambda(s,a)+\alpha} \sum_{s' \in \states} \transP(s,a,s') \cdot  \valueq(s',\sigma(s')).
\]
% Equation for the expected discounted reward for a schedule $\sigma$ from state $s$ can be reduced to the following equation~\cite[Eq 11. 5. 4]{puterman2014markov}:

% \track{The discounted reward $\discobjective^{\ctmc} (\alpha)(s)$~\cite[Eq 11.5.4]{puterman2014markov} for a pure schedule $\sigma$ from $s\in\states$, can be reduced to the following equation:
% \begin{equation*}
%     \begin{split}
% \rho(s,a_{\sigma}) + \frac{\lambda(s,a_{\sigma})}{\lambda(s,a_{\sigma})+\alpha} \sum_{s' \in s} P(s,a_{\sigma},s') \cdot \discobjective_{\sigma}^{\Mdp}(\alpha)(s') ,
%     \end{split}
% \end{equation*}
% where $a_\sigma = \sigma(s, a)$ is the action taken under the schedule $\sigma$ from state $s$ and $P(s,a_{\sigma},s')$ is the probability of taking a transition from state $s$ to $s'$ on action $a_{\sigma}$.
% Here the term $\frac{\lambda(s,a_{\sigma})}{\lambda(s,a_{\sigma})+\alpha}$ is derived by applying the continuous time discounting to the expected time delay in state $s$.} 

% \track{These equations motivate the definition of the $Q$-values.}\todo{The part in red can be removed.}
It gives the total expected discounted reward obtained by taking action $a$ from $s$, and following $\sigma$ afterwards. 
The optimal Q-function, denoted $\mathcal{Q}^*$ is given by,
 \[
 \rho(s,a) + \frac{\lambda(s,a)}{\lambda(s,a)+\alpha} \sum_{s'\in\states}\transP(s,a,s') \cdot \max_{a'\in\actions} \mathcal{Q}^{*}(s',\alpha').
 \]
 Q-learning uses stochastic approximation~\cite{Sutton18} to estimate the $\mathcal{Q}^*$ function. When a transition from state $s$ to $s'$ on an action $a$ with delay $\tau$ is observed, the $\Qf$ estimates are updated as~\cite[Eq 12]{BD94}: 
%  following Q-learning rule for CTMDPs is obtained,
%  \begin{equation*}
%     \begin{split}
%         \newvaluef(s,a) = \kvaluef(s,a) + \rho_{k}[ \frac{1-e^{-(\lambda(s,a) + \alpha)\cdot \tau}}{\lambda(s,a)}r_{s}(s,a,s') + \\
%         e^{-\alpha\tau} \max_{a'} \kvaluef(s',a') - \kvaluef(s,a)]
%     \end{split}
% \end{equation*}
% \newvaluef(s,a) = \kvaluef(s,a) + \rho_{k}[ \frac{1-e^{-(\lambda(s,a) + \alpha)\cdot \tau}}{\lambda(s,a)}r_{s}(s,a,s') + \\
%         e^{-\alpha\tau} \max_{a'} \kvaluef(s',a') - \kvaluef(s,a)]
% 
 \begin{multline*}
        \newvaluef(s,a) := (1-\beta_k)\kvaluef(s,a) + \\ \beta_{k} \Big( r(s,a,s') +
        e^{-\alpha\tau} \max_{a'} \kvaluef(s',a') \Big),
\end{multline*}
where $r(s,a,s')$ is the sampled reward from state $s$ to $s'$, the sampled transition time is $\tau$, and $\beta_{k}$ is the learning rate. 
% The Q-function is updated with respect to the learning rate, i.e, the extent to which the updation occur with respect to the sampled value (second term in the equation) is determined by how large the learning rate is.
% The RL algorithm samples through states and updates the Q-function iteratively.
% An episode is a finite sequence of states, actions and rewards which terminates on a certain specified condition like user defined episode length.
 The RL algorithm samples through states and updates the Q-function iteratively.
 While sampling, the agent picks the action based on an RL schedule.
% The RL algorithm conducts a number of episodes and updates the Q-function iteratively. 
% The action is chosen according to an  RL policy. 
% There are many different effective RL policies for getting accurate Q-values \cite{watkins1989learning}\cite{auer2002using}.
The optimal schedule is generated after completion of some number of episodes by taking the action that gives the highest Q-value from each state.

We focus on how to \emph{automatically obtain reward mechanisms for $\omega$-regular objectives for CTMDPs} so that off-the-shelf RL algorithms can learn an optimal schedule.

\section{Problem Statement}
\label{sec:p_statement}
\paragraph{Omega-regular Objectives.}
An $\omega$-regular objective is defined by a nondeterministic B\"uchi automaton $\oautomata = (\alphabets, \ostate, \oinitstate, \otrans, F)$ where $\alphabets$ is a finite \emph{alphabet}, $\ostate$ is a finite set of \emph{states}, $\oinitstate \in \ostate$ is an \emph{initial state}, $\otrans : \ostate \times \alphabets \rightarrow 2^\ostate$ is a \emph{transition function} and $F \subseteq \ostate$ is the set of \emph{accepting states}.
A B\"uchi automaton is deterministic, if $\delta(q, a)$ is singleton for all $(q, a) \in \ostate{\times}\alphabets$.
We define the extended transition function $\hat{\otrans}: \ostate \times \alphabets^* \rightarrow 2^\ostate$, derived from $\otrans$, as 
% \[
% \hat{\otrans}(q, w) = \begin{cases}
% \set{q} & \text{ if $ w = \varepsilon$}\\
% \bigcup\limits_{q' \in \otrans(q, a)} \hat{\otrans}(q', x) & \text{ if } w = ax \text{ for } a {\in} \Sigma, x {\in} \Sigma^*.
% \end{cases}
% \]
 $\hat{\otrans}(q, \varepsilon) = \set{q}$ and $\hat{\otrans}(q, ax) = \cup_{q' \in \otrans(q, a)} \hat{\otrans}(q', x)$, for $q \in \ostate$ and $ax \in \Sigma\Sigma^*$.

% B{\"u}chi automaton by \textbf{DBW}.

A \emph{run} $r$ of $\oautomata$ is an infinite sequence $(r_0, w_0, r_1, w_1, \ldots)$ where $r_0 = q_0$, $r_i \in \ostate$, $w_i \in \alphabets$ and $r_{i+1} \in \otrans(r_i,w_i)$ for all $i \in \mathbb{N}$. The word of a run $r = (r_0, w_0, r_1, w_1, \ldots)$ is ${\sf L}(r)=(w_0 w_1 \cdots)$ . Let the set of runs of $\oautomata$ be $\oruns$.
We say that a run $r \in \oruns$ is accepting if there exists a $q_f \in F$ such that $q_f$ occurs infinitely often in $r$. An $\omega$-word $w = (w_0 w_1 \cdots)$ is accepted by $\oautomata$ if there exists an accepting run $r_w = (r_0, w_0, r_1, w_1, \ldots)$ of $\oautomata$. 
The language of the automaton $\oautomata$, denoted $\Ll(\oautomata)$ is the set of all words that is accepted by the automaton. 

\paragraph{CTMDPs and Omega-regular Objectives.}
In order to express the properties of a CTMDP $\Mdp$ using a B\"uchi automaton, we introduce the notion of a labelled CTMDP.
A \emph{labelled} CTMDP is a triple $(\Mdp, \atomicprop, \labelling)$ where $\Mdp$ is a CTMDP, $\atomicprop$ is a set of atomic propositions, and $\labelling: \states \rightarrow 2^\atomicprop$ is a labelling function. 
Let $\oautomata = (2^{\atomicprop}, \ostate, \oinitstate, \otrans, \acceptingc)$ be a B\"uchi automaton expressing the learning objectives of $\Mdp$.

Recall that for a CTMDP $\Mdp$ under a schedule $\sigma$ we write $X_n$, $Y_n$, $D_n$, and $T_n$ for the random variables corresponding to the $n$-th state, action, time-delay at the $n$-th state, and time-stamp (total time spent up to the $n$-th state).
We introduce the random variable $F_n$ to indicate if the sequence of observations of the CTMDP leads to an accepting state on $\oautomata$ in $n$-steps, i.e., 
$
F_n = [\hat{\delta}(L(X_0)\cdot L(X_1) \cdots L(X_n)) \cap F].
$

For a CTMDP $(\Mdp, \atomicprop, \labelling)$ and  automaton $\oautomata = (2^{\atomicprop}, \ostate, \oinitstate, \otrans, \acceptingc)$,
we study the following problems:
\begin{enumerate}
    \item {\bf Satisfaction Semantics.} 
    Compute a schedule of $\Mdp$ that maximizes the probability of visiting accepting states $F$ of $\oautomata$ infinitely often.
    We define the satisfaction probability of a schedule $\sigma$ from starting state $s$ as: 
\begin{equation*}
\PSemSat^{\Mdp}_{\oautomata}(s, \sigma) 
{=}   {\Pr}^{\Mdp}_{\sigma}(s) \set{\forall_i  \exists_{j{\geq} i} F_j }.
\end{equation*}
% \begin{equation*}
% \PSemSat^{\Mdp}_{\oautomata}(s, \sigma) 
% {=}   \Pr{}^{\Mdp}_{\sigma}(s) \set{ r \in \Runs_{\sigma}^{\Mdp}(s) \colon
%   L(r) \in \Ll(\oautomata) }.
% \end{equation*}
% i.e, the probability of runs from $s$ under $\sigma$ for which its corresponding run in $\oautomata$ visits an accepting state infinitely often.
Intuitively, it describes the probability of runs from state $s$ under $\sigma$ in the CTMDP such that the corresponding run in $\oautomata$ visits the accepting states infinitely often.
The optimal satisfaction probability
$\PSemSat^{\Mdp}_{\oautomata}(s)$ for  $\oautomata$ 
is defined as $\sup_{\sigma \in \Sigma_{\Mdp}} \PSemSat^{\Mdp}_{\oautomata}(s, \sigma)$, and we say
that a schedule $\sigma \in \Sigma_\Mdp$ is an optimal schedule for $\oautomata$ if
$\PSemSat^{\Mdp}_{\oautomata}(s, \sigma)  = \PSemSat^{\Mdp}_{\oautomata}(s)$ for all $s \in \states$.

\item {\bf Expectation Semantics.} Compute a schedule of $\Mdp$ that maximizes the long-run expected average time spent in the accepting states of $\oautomata$.
We define the expected satisfaction time of a schedule $\sigma$ from starting state $s$ as: 
\begin{equation*}
\ESemSat{}^{\Mdp}_{\oautomata}(s, \sigma) 
=   \mathbb{E}^{\Mdp}_{\sigma}(s) \set{ \liminf_{n \rightarrow \infty} \frac{\sum_{i=1}^{n} F_i \cdot D_i}{T_n}}.
\end{equation*}
The optimal expected satisfaction time
$\ESemSat^{\Mdp}_{\oautomata}(s)$ for specification $\oautomata$ is defined as $\sup_{\sigma \in \Sigma_{\Mdp}} \ESemSat^{\Mdp}_{\oautomata}(s, \sigma)$, and we say that $\sigma \in \Sigma_\Mdp$ is an optimal expectation maximisation schedule for $\oautomata$ if $\ESemSat^{\Mdp}_{\oautomata}(s, \sigma) = \ESemSat^{\Mdp}_{\oautomata} (s)$.
\end{enumerate}

\paragraph{Product Construction.} 
Given a \emph{labelled} CTMDP $(\Mdp, \atomicprop, \labelling)$ where $\atomicprop$ is a set of atomic propositions, and $\labelling: \states \rightarrow 2^\atomicprop$ is a labelling function, and a   B{\"u}chi automaton $\oautomata = (2^{\atomicprop}, \ostate, \oinitstate, \otrans, \acceptingc)$, the product CTMDP is defined as $\Mdp \times \oautomata = ((\states \times \ostate),(\initstate,\oinitstate),\actions,\transR^{\times},\acceptingc^{\times})$ where the rates are $\transR^{\times} : (\states \times \ostate) \times \actions \times (\states \times \ostate) \rightarrow \Reals_{\geq 0}$ such that $\transR^{\times} ((s,q),a,(s',q')) = \transR(s,a,s')$ if $\transR(s,a,s')>0$ and $\otrans (q,\labelling(s)) = \{q'\}$.
If $F$ is the set of accepting 
% transitions 
states
in $\oautomata$, then the accepting condition is a set 
% transitions $F^\times$ where $((s,q),a,(s',q')) \in F^\times$ iff $(q,\labelling(s),q') \in F$ 
$F^\times$ of states where $(s,q) \in F^\times$ iff $q \in F$. 
 An example of a product CTMDP is given in Appendix~\ref{sat vs expt}.

\paragraph{Good-for-CTMDP Automata.} 
From the definition of both the semantics, it is clear that the optimal schedule requires some memory to monitor the run in the B\"uchi automaton (see Example~\ref{ex:mem} in Appendix~\ref{exmp:memory}).
% It is easy to see that for both optimization problems, the optimal schedules require memory.
For the right kind of B\"uchi automata~\cite{HPSS20}, the amount of memory required can be equal to the size of the automata. 
A key construction to compute these schedules is the product construction, where the CTMDP and the automaton are combined together as a CTMDP with accepting states governed by the accepting states of the B\"uchi automata.
On the other hand, not every B\"uchi automaton can be used for this construction.
The class of B\"uchi automata where the semantic value of satisfaction of the property on the MDP equals to the corresponding problems on the product structure, are called good-for-MDP (GFM) automata~\cite{HPSS20}.

We introduce the notion of good-for-CTMDP automata in Appendix~\ref{app:gfm}.
% We also argue that every $\omega$-regular property can be expressed as a good-for-CTMDP automaton\todo{add a reference}. 
If a B\"uchi automaton is GFM, then one can show via uniformization that it is also good-for-CTMDPs.
There exist several syntactic characterizations of good-for-MDP automata including suitable limit-deterministic B\"uchi automata (SLDBA)~\cite{sickert2016limit} and slim automata~\cite{HPSS20}.
Moreover, every LTL specification can be effectively converted into a GFM B\"uchi automata.
Moreover, there exist tools (OWL and Spot) to convert LTL objectives to good-for-CTMDP automaton.
Hence, in this paper, w.l.o.g., we assume that $\omega$-regular objectives are given as good-for-CTMDP automata.

\paragraph{Problem Definition.} Given a CTMDP $\Mdp$ with unknown transition structure and rates, and an $\omega$-regular objective $\phi$ given as a good-for-CTMDP B\"uchi automata $\oautomata$, we are interested in the following reward translation problem for the satisfaction semantics and for the expectation semantics. 

\begin{problem}[Reward Translation Scheme]
    Design a reward scheme for $\oautomata$ 
    % for satisfaction (expectation) semantics 
    such that any off-the-shelf RL algorithm optimizing the discounted reward in CTMDPs converges to an optimal schedule for satisfaction (expectation) semantics.
\end{problem}

In Section~\ref{sec:theorems&algo} we provide a solution for the satisfaction semantics, while in Section~\ref{sec:d_time} we sketch a solution for this problem for the expectation semantics.
% In both of these results, 
We reduce these problems to average reward maximization for CTMDPs.
Since average-reward RL algorithms for CTMDPs and MDPs require strong assumptions on the structure (such as communicating MDPs)~\cite{Sutton18}, we solve the average-reward RL problem by reducing it to a discounted-reward problem using the following result.
\begin{theorem}\label{corollary:1}
For every CTMDP $\Mdp$, there exists a pure schedule $\bschedule$ and a threshold $0 {\leq} \ctmdpthrate^{\Mdp} {<} 1$ such that
% \begin{inparaenum}[(1).]
    % \item 
    for every discount-rate function $\ctmdprate$, where $\ctmdprate(s,a) \geq \ctmdpthrate^{\Mdp}$ for every valid state-action pair $(s,a)$, the schedule $\bschedule$ is an optimal schedule maximising the expected discounted reward.
    % \item 
    Moreover, $\bschedule$ also maximizes the expected average reward.
% \end{inparaenum}
\end{theorem}
This schedule $\bschedule$ is known as a Blackwell optimal schedule.
We provide a novel uniformization based proof for this theorem in Appendix~\ref{sec:blackwell_optimality}.
We show that we need different reward translation schemes for the two semantics.

% In the next section, we provide a simple uniformization based proof for the Blackwell optimality for CTMDPs.
% \track{Add Theorem 3 here?}

\section{RL for Satisfaction Semantics}
\label{sec:theorems&algo}
% \vspace{0.5em}\noindent\textbf{Sub-MDP and End-Components.} For a given MDP $\Mdp= (S, s_0,\actions, T)$, a {\it sub-MDP} of $\Mdp$ is an MDP $\Mdp' = (S', s_0',\actions', T')$, where $S' \subset
% S$, $\actions' \subseteq \actions$ is such that $\actions'(s) \subseteq \actions{}(s)$,
% % \track{{\sf Act} in prelims instead of $A$} for every $s \in S'$,  
% and $T'$ is analogous to $T$ when restricted to $S'$ and
% $A'$. Moreover $\Mdp'$ is closed under probabilistic transitions.
% An {\it end-component}~\cite{alma991027942769706011} of an MDP $\Mdp$ is a sub-MDP $\Mdp'$ such that for every pair of states $s, s' \in S'$, there is a 
% schedule that can reach $s'$ from $s$ with positive probability. 
% A maximal end-component is an end-component that is maximal under set-inclusion.
% An end-component is \emph{winning} if it contains an accepting state.
% Every state $s$ of an MDP $\Mdp$ belongs to at most one maximal end-component.
% We can define sub-CTMDP, and end components in a CTMDP, as in the case of MDPs.
% % Further, we can see that Theorem~\ref{thm:ec} 
% Following is a property of end-components in an MDP that also holds for CTMDPs. 

%  \begin{theorem}[End-Component Properties \cite{alma991027942769706011}]\label{thm:ec}
% Once an end-component $C$ of an MDP is entered, there is a schedule that visits every state-action pair in $C$ with probability 1 and stays in $C$ forever. Moreover, for every schedule the probability that a run ends up in an end-component is 1.  
% \end{theorem}

% \vspace{0.5em}\noindent\textbf{Augmented Product CTMDP.} 
We reduce the problem of satisfaction semantics of an $\omega$-regular objective in a CTMDP to an expected average reward objective.
Using Blackwell optimality result stated in Theorem~\ref{corollary:1}, we further reduce this to an expected discounted reward objective which allows us to use off-the-shelf RL for CTMDP for learning schedules for $\omega$-regular objectives.

To find a schedule satisfying an $\omega$-regular objective in a CTMDP, we need to identify the accepting end-components where an accepting end-component~\cite{alma991027942769706011} is a sub-MDP that is closed under probabilistic transitions and contains an accepting state.
It is well known~\cite{alma991027942769706011} that as an end-component $C$ of an MDP is entered, there is a schedule that visits every state-action pair in $C$ with probability 1 and stays in $C$ forever.
Hence, a schedule that maximizes the probability of satisfaction of a given $\omega$-regular objective maximizes the probability of reaching the accepting end-components.
In Appendix~\ref{app:EC}, we show an example of a CTMDP and its end-components.
Also the MDP in the top part of Figure~\ref{fig:p3}, is itself an accepting end-component since the state $q_0$ is accepting.

% To tackle this problem, 
We further reduce the problem to an average reward problem as described below and then specify a reward function such that the schedule maximising the expected average reward maximizes the probability of satisfying the objective.
% Note that a reachability objective does not require identifying the end-components.
\paragraph{Reduction to Average Reward.}
Before describing our RL algorithm for unknown CTMDP, we first describe the reduction when an input CTMDP is fully known to explain the intuition behind our algorithm.
Consider a CTMDP $\Mdp$, a \textbf{GFM} $\oautomata$, and let $\pmdp$ denote the product CTMDP.
For our reduction, we define a constant $\zeta \in (0,1)$ and an augmented product CTMDP, denoted by $\augmdp$. 
The CTMDP $\augmdp$ is constructed from $\pmdp$ by adding a new sink state $t$ with a self loop labelled by an action $a'$ and with rate $\lambda(t,a') > 0$, and making it the only accepting state in $\augmdp$. 
Further, in $\augmdp$, the rates of each outgoing transition from an accepting state in $\pmdp$ is multiplied by $\zeta$. 
Also, for each action $a$ from an accepting state $s$ in $\pmdp$, in $\augmdp$ we add a new transition to the sink state $t$ with rate $\lambda(s,a) \cdot (1-\zeta)$ where $\lambda(s,a)$ is the exit rate of the state-action pair $(s,a)$ in $\pmdp$. 
% We also add an action $a'$ from $t$ with a single transition to itself.
% The rate of this action can be any positive constant $\lambda(t,a')$.
% as there are no other transition from $t$ and if the system reaches $t$ it will stay there forever.    
Figure~\ref{fig:p3} shows an example of this construction.
% \track{WHAT IS $\lambda(t,a')$? NEEDS TO BE ADDED.}
Note that in the figure, $q_0$ is the only accepting state in the product CTMDP.
There are two outgoing transitions from $q_0$ on action $a$ to $q_1$ and $q_2$ with rates $r_1$ and $r_2$ respectively, and hence $\lambda(q_0, a) = r_1 + r_2$.
% We add a sink $t$ which is the only accepting state in $\Mdp^\zeta$.
We then add a transition from $q_0$ to $t$ with rate $(r_1+r_2) \cdot (1-\zeta)$.

\begin{figure}[t]
 \centering
  \begin{minipage}{0.3\textwidth}
     \begin{tikzpicture}[shorten >=1pt, node distance=2.3 cm, on grid, auto,thick,initial text=]
\begin{scope}[every node/.style={scale=1}]
\node (l0) [state,accepting, fill = safecellcolor] {$q_0$};
\node (l1) [state,right = of l0, fill = safecellcolor]   {$q_1$};
\node (l2) [state,left = of l0,fill = safecellcolor]   {$q_2$};
\end{scope}
 \begin{scope}
\path [-stealth, thick]
    (l0) edge [bend left]  node [above] {$a,r_1$}   (l1)
    (l0) edge [bend right]   node [above] {$a,r_2$}   (l2)
    (l1) edge [bend left]  node [below] {$b,r_3$}   (l0)
    (l2) edge  [loop above] node [above] {$c,r_4$}   ()
    (l2) edge [bend right] node [below] {$d, r_4$} (l0)
    ;
\end{scope}
\end{tikzpicture}
  \end{minipage}
  \begin{minipage}{0.3\textwidth}
    \centering
     \begin{tikzpicture}[shorten >=1pt, node distance=2.3 cm, on grid, auto,thick,initial text=,]
\begin{scope}[every node/.style={scale=1}]
\node (l0) [state,fill = safecellcolor] {$q_0$};
\node (l1) [state,right = of l0, fill = safecellcolor]   {$q_1$};
\node (l2) [state,left = of l0,fill = safecellcolor]   {$q_2$};
\node (l3) [state, above = of l0,accepting, fill = goodcellcolor] {$t$};
\end{scope}
 \begin{scope}
\path [->]
    (l0) edge [bend left]  node [above] {$a,r_1  \cdot \zeta$}   (l1)
    (l0) edge  [bend right] node [above] {$a,r_2 \cdot \zeta$}   (l2)
    (l1) edge [bend left]  node [below] {$b,r_3$}   (l0)
    (l2) edge  [loop above] node [above] {$c,r_4$}   ()
    (l0) edge node [right] {$a,r_1+r_2 \cdot(1-\zeta)$}(l3)
    (l3) edge [loop left] node [left] {$a', \lambda(t,a')$} () 
    (l2) edge [bend right] node [below] {$d, r_4$} (l0)
    ;
\end{scope}
\end{tikzpicture}
  \end{minipage}
  \caption{A product CTMDP $(\Mdp \times \oautomata)$ (top) and its corresponding augmented product CTMDP $\Mdp^{\zeta}$ (bottom).}
  \label{fig:p3}
\end{figure}
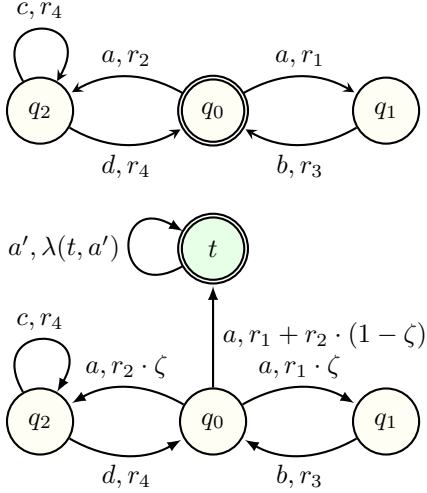

% As we can see, the set of states in $\pmdp$ and $\augmdp$ differ only by $t$ and any run that reaches $t$ will stay there. 
% Therefore, given a schedule $\sigma$ in $\augmdp$, the corresponding schedule in $\pmdp$ is defined in a way where the action chosen from any state in $\pmdp$ will be the same as that of $\sigma$.
With a slight abuse of notation, if $\sigma$ is a schedule in the augmented CTMDP $\Mdp^\zeta$, then we also denote by $\sigma$ a schedule in $\Mdp \times \oautomata$ obtained by removing $t$ from the domain of $\sigma$.
Thus fix a schedule $\sigma$ in both $\Mdp^\zeta$ and in $\Mdp \times \oautomata$.
Note that for every state in an accepting end-component, the probability of reaching the sink $t$ in $\Mdp^\zeta$ is $1$.
Similarly, for every state in a rejecting end-component, the probability of reaching $t$ in $\Mdp^\zeta$ is $0$.
The probability of reaching $t$ in $\Mdp^\zeta$ under $\sigma$ overapproximates the probability of reaching the accepting end-components in $\Mdp \times \oautomata$ under $\sigma$.
The difference in the two probabilities occurs since in $\Mdp^\zeta$, from the transient accepting states, with probability $1-\zeta$, one can reach the sink $t$.
This approximation error tends to $0$ as $\zeta$ tends to $1$.
We define a reward function in $\augmdp$ such that a schedule maximising the expected average reward in $\augmdp$ maximizes the probability of satisfying the $\omega$ regular objective in $\pmdp$.

\paragraph{Reward Function.} 
The reward function 
% is such that every time the sink state is reached, we 
provides a reward of $1$ per time unit for staying in the accepting sink $t$, while the reward is $0$ otherwise, i.e.
% \[
% rew(s) = [s=t]
% \]
\[
rew(s) = \begin{cases}
1 & \text{if $s=t$}  \\
0 & \text{otherwise}
\end{cases}
\]
As there is only a single action $a'$ from state $t$ in $\augmdp$ which is a self loop, we can conclude that any schedule that maximizes the probability of reaching $t$ also maximizes the expected average reward in $\mathcal{M}^{\zeta}$.
Further, following the discussion above, for high values of $\zeta$, the schedule also maximizes the probability of satisfying the $\omega$-regular objective in $\Mdp \times \oautomata$.
We thus have the following.

\begin{theorem}\label{theorem:4.1}
There exists a threshold $\zeta' \in (0,1)$ such that for all $\zeta > \zeta'$, and for every state $s$, a schedule maximising the expected average reward in $t$ in $\augmdp$ is 
\begin{inparaenum}[(1)]
\item an optimal schedule in the product CTMDP $\mathcal{M} \times \mathcal{A}$ from $s$ for satisfying the $\omega$-regular objective $\phi$.
Further, since $\oautomata$ is a GFM, we have that \item $\sigma$ induces an optimal schedule for the CTMDP $\mathcal{M}$ from $s$ with objective $\phi$.
\end{inparaenum}
\end{theorem}
Detailed proof of this theorem is provided in Appendix~\ref{app:sat}.
From the above theorem, we have that for a large $\zeta$ value, a schedule maximising the expected average reward in $\augmdp$ also maximizes the probability of satisfying the $\omega$-regular property in $\pmdp$.
Therefore, when the CTMDP is known, the problem of satisfaction semantics of an $\omega$-regular property is reduced to an expected average reward objective.
\paragraph{The case of unknown CTMDP.} 
% Now we look at the case when the CTMDP is unknown, 
Recall that we consider a CTMDP model with unknown rate and transition structure.
For such unknown CTMDP models, 
% it may not be possible to identify the winning end-components a priori.
% Thus for RL, we 
an RL algorithm cannot construct the product $\Mdp \times \oautomata$ explicitly.
 From Theorem \ref{theorem:4.1}, we can conclude that any schedule maximising the expected average reward that is accrued by visiting the sink state $t$ in $\augmdp$ where $\zeta>\zeta'$ for some $\zeta' \in (0,1)$ also maximizes the probability of satisfying the $\omega$-regular objective $\phi$ in $\pmdp$.
% \todo{We need to rewrite this para.} 
% But, as we do not know the transition structure of the CTMDP beforehand, we cannot construct an augmented product CTMDP for finding an optimal schedule. Instead, we define a reward function $rew$ where 
This leads to a very simple model-free RL algorithm which does not require the augmented product CTMDP $\Mdp^\zeta$ to be constructed explicitly.
% since the transition structure of the CTMDP is not known beforehand. 
% without constructing the augmented product CTMDP, 
We define the following reward function $rew'$ to be used by the RL algorithm: 
\[
rew'((s,q),a) = \begin{cases}
1 \text{\quad with probability $1-\zeta$ if $(s,q)$ is} \\
\text{\quad \quad accepting} \\
0  \text{\quad otherwise}
\end{cases}
\]
% which gives a positive reward with probability $1 - \zeta$ on each action from an accepting state in $\pmdp$ and $0$ reward for any action from other states. 
% where $C$ is a large positive constant. 
Recall that in the augmented product $\Mdp^\zeta$, for each action from an accepting state, we add a transition to sink state $t$ with probability $1-\zeta$, and give a reward of $1$ for staying in $t$ per unit time.
The RL algorithm simulates this in the following way: When a transition from an accepting state is visited, the learning agent tosses a biased coin and obtains a reward of $1$ with probability $1-\zeta$.
% for each action taken from the accepting state.
% As we cannot construct the augmented product CTMDP $\augmdp$, and thus cannot give positive rewards for each time unit spent in the sink, we provide a positive reward of $1$ with probability $1-\zeta$ for each action taken from an accepting state.
% This simulates that when we visit an accepting state, with probability $1-\zeta$ we reach the sink state $t$ in the augmented product CTMDP $\augmdp$.
% This reward function gives a large positive reward from each action from an accepting state with probability $1-\zeta$ which is similar to what is done by the reward machine .
Therefore, any schedule maximising the expected average reward w.r.t. $rew'$ also maximizes the probability of satisfying the objective.
% Thus, a schedule maximising the expected average reward also maximises the probability of satisfying the objective.
As Theorem~\ref{corollary:1} shows the existence of Blackwell optimal schedules in CTMDPs, we can conclude that for a high enough discount factor, any off-the-shelf model-free RL algorithm for CTMDP maximising the expected discounted reward gives an optimal schedule maximising the satisfaction of $\phi$. 
% We provide an algorithm based on this reward function in Appendix~\ref{algo}.
A pseudocode of our algorithm is given in Appendix~\ref{algo}.

% \paragraph{Algorithm for Satisfaction Semantics.} The algorithm is similar to that in Algorithm $\ref{algo:expt}$, but here the reward function and the terminating condition would be different.
% The major differences with Algorithm~\ref{algo:expt} are in the definition of the reward function $rew$
% % \track{USE $rew$?} 
% and how we terminate an episode. 
% The reward function $rew$ is defined as shown above and an episode ends when a positive reward is obtained.
% A detailed algorithm is provided in Appendix~\ref{algo: sat}
% With this reward function, any off the shelf model-free RL algorithm for CTMDP maximising the discounted payoff gives the optimal schedule for reaching $t$ and hence in turn gives the optimal schedule maximising the satisfaction of $\phi$.

\section{RL for Expectation Semantics}
\label{sec:d_time}
\begin{table*}[t!]
  \centering
  %\small
  \begin{tabular}[c]{l|cccccccccccccc}
Name & states & prod. & Sat. Prob. & Est. Sat. & Time 1 & Exp. Prob. & Est. Exp. & Time 2\\\hline
\texttt{RiskReward}  & 4 & 8 & 1 & 1    & 1.713 & 0.9 & 0.9 & 0.967\\
\texttt{DynamicPM-tt\_3\_qs\_2}  & 816 & 825 & 1 & 1    & 3.586 & 1 & 1 & 3.62\\
\texttt{QS-lqs\_1\_rqs\_1\_jt\_2}  & 266 & 282 & 1     & 1    & 3.401 & 1 & 1 &3.486 \\
 \texttt{QS-lqs\_1\_rqs\_1\_jt\_5}  & 3977 & 4152 & 1     & 1    &  5.482 & 1 & 1 & 5.524\\
\texttt{QS-lqs\_2\_rqs\_2\_jt\_3}  & 11045 & 24672 & 1     & 1 & 15.158 & 1 & 1 &  15.395 \\
\texttt{ftwc\_001\_mrmc}  & 82 & 122 & 0.999779 & 0.999779    &  94.288 &0.999779 & 0.999779  &98.256\\
% \texttt{PollingSystem-jt1\_qs1}  & 16 & 20 & 1     & 1    & 78.00 &18.8849 & 1.12183\\
\texttt{PollingSystem-jt1\_qs4}  & 348 & 352 & 1     & 1    & 3.423 & 1 & 1 &3.421\\
\texttt{PollingSystem-jt1\_qs7}  & 1002 & 1006 & 1     & 1    & 3.576 & 1 & 1 &3.580 \\
% \texttt{ErlangStages-k500\_r10}  & 508 & 509 & 1 & 1    & 2.73944 & 5.786 & 1.12393 \\
% \texttt{ErlangStages-k2000\_r10}  & 2008 & 2009 & 1 & \textcolor{red}{0.5} & 3.23 & 1 & \textcolor{red}{0.5} & 3.241 \\
\texttt{SJS-procn\_6\_jobn\_2}  & 17 & 21 & 1     & 1    & 3.253 & 1 & 1 & 3.257\\
\texttt{SJS-procn\_2\_jobn\_6}  & 7393 & 7405 & 1     & 1    & 4.336 & 1 & 1 & 4.234
 \end{tabular}
 \caption{\label{tab:experiment}Q-learning results.  The default values of the learner
    hyperparameters are: $\zeta = 0.99$ (for satisfaction semantics), $\epsilon=0.1$ (used in picking $\epsilon$-greedy actions in Q-learning), $\beta=0.01$ (learning rate),
    tol$=0.01$ (tolerance for numerical approximation), ep-l$=300$ (episode length), and ep-n$=20000$ (episode numbers).  Times are in seconds.}
\end{table*}

We study the expectation semantics of $\omega$-regular objective and show that the problem can be reduced to maximising the expected average reward problem in CTMDPs.
Using Theorem~\ref{corollary:1}, this reduces to maximising expected discounted reward for a large discount factor.
We then describe the corresponding reward machine to maximize the \emph{expected satisfaction time} in the good states.

\paragraph{Reduction to Average Reward.} For an $\omega$-regular objective $\phi$, let $\oautomata$ be a \textbf{GFM} corresponding to $\phi$ with a set $F$ of B{\"u}chi accepting states. Let $\Mdp$ be a CTMDP and $\Mdp \times \oautomata$ be the product CTMDP of $\Mdp$ and $\oautomata$.
For a state $s$ in $\Mdp \times \oautomata$, we define the \emph{expected satisfaction time} of a schedule $\sigma$ from starting state $s$ as:
% For a given finite run $r_f = (s_1,a_1,t_1,s_2,a_2,t_2...a_{n-1},t_{n-1},s_n)$, we define the residence time of $s$ in $r_f$ as 
% $\dtime(s)(r_f) = \frac{\sum_{s_i = s,i=1}^{n-1} t_i}{\sum_{j=1}^{n-1} t_j}$.
% Similarly for an infinite run $r_{inf}=(s_1,t_1,a_1,s_2,t_2,a_2...)$, we define the residence time of $s$ in $r_{inf}$ as $\dtime(s)(r_{inf}) = \liminf_{n \rightarrow \infty} \frac{\sum_{i=1,s_i = s}^{n} t_i}{\sum_{j=1}^{n}t_j}$. Note that we consider $\liminf$ since the limit may not exist in general.
 %\todo{Should we use equation environment}
\[
\ESat^{\Mdp {\times} \oautomata}_{\sigma}(s) {=} \mathbb{E}^{\Mdp \times \oautomata}_{\sigma}(s) \set{\liminf_{n \rightarrow \infty} \frac{\sum_{i=1}^{n} [X_i \in F^{\times}] {\cdot} D_i}{T_n}}.
\]
It gives the long-run expected average time spent in the accepting states. 
The reward rate function $r' : S \rightarrow \{0,1\}$ for $~{\Mdp \times \oautomata}$ is defined such that $r'(s) = 1$ if $s \in F^{\times}$, and $r'(s) = 0$, otherwise. Thus the reward is $r'(s)\cdot t = t$ for $s \in F^{\times}$ if $t$ time is spent in $s$.

% The following lemma gives an equivalence between the residence time and the average reward obtained for every run in $(\Mdp \times \oautomata)^{\sigma}$ for every schedule $\sigma$.
% \begin{lemma} \label{lemma:6.1}
% For a product CTMDP $\Mdp \times \oautomata$ with $\oautomata$ being a \textbf{GFM} for an $\omega$-regular objective, and a set $T$ of target states, for a schedule $\sigma$, for all runs in $(\Mdp \times \oautomata)^{\sigma}$, the average reward obtained by the reward function $r'$ is equal to the residence time in $(\Mdp \times \oautomata)^{\sigma}$.
% \end{lemma}

% For a schedule $\sigma \in \Sigma_{\Mdp \times \oautomata}$, we denote the expected residence time in $T$ under $\sigma$ by $\edtime(T)$, and the expected average reward obtained under $\sigma$ for the reward function $r'$ by $\eavgreward(T)$.
% From Lemma~\ref{lemma:6.1} it follows that for any schedule $\sigma$, the expected average reward obtained by $r'$ and the expected residence time in $T$ are equivalent.
The following lemma (proof in Appendix~\ref{app:expt}) gives an equivalence between the expected satisfaction time and expected average reward obtained in $\pmdp$.
\begin{lemma} \label{lemma:6.2}
For a product CTMDP $\Mdp \times \oautomata$ where $\oautomata$ is a \textbf{GFM} for an $\omega$-regular objective and for a schedule $\sigma$, the expected average reward obtained w.r.t. the reward function $r'$ is equal to the expected satisfaction time in $~{(\pmdp)}$ and there exists a pure schedule that maximizes this.
\end{lemma}
% For the product CTMDP $\Mdp \times \oautomata$ with accepting states $T$ corresponding to the B{\"u}chi acceptance, 
% Our objective is to find an optimal schedule $\sigma^*$ that maximises the expected residence time in $T$, i.e, $\edstime(T) = \sup_{\sigma \in \Sigma_{\Mdp \times \oautomata}}\edtime(T)$. We show that such a schedule exists in $\pmdp$, and there exists an optimal schedule that is pure.
% From \cite{puterman2014markov}, we know that there exist optimal pure schedules for maximising the expected average reward and from this result we get the following lemma.

% \begin{lemma}\label{lemma:6.3}
% For a product CTMDP $\Mdp \times \oautomata$ there exists a pure schedule that maximises the expected satisfaction time.
% \end{lemma}
Using the results from Lemma~\ref{lemma:6.2} and Theorem~\ref{corollary:1}, we can conclude that a schedule maximising the discounted reward objective for a large discount factor in $\pmdp$ with reward function $r'$ also maximizes the expected satisfaction time.
% Lemma \ref{lemma:6.3} follows from Lemma~\ref{lemma:6.2} since it is known that an optimal pure schedule exists for maximising $\eavgreward(T)$ over all $\sigma \in \Sigma_{\Mdp}$~\cite{puterman2014markov}.
% we can conclude that there exists an optimal pure schedule $\sigma^*$ for the expectation semantics. 

\paragraph{Algorithm for Expectation Semantics.}
Here, we provide a brief description of the algorithm.  
% We show the procedure to obtain an optimal schedule $\sigma$ for expectation semantics in Algorithm \ref{algo:expt}. 
% Initially, the Q-function for reinforcement learning is initialised to zeroes. 
% Line 1 initialises the Q-function for reinforcement learning.
The Q-function is defined on the states of the product CTMDP, i.e, $\mathcal{Q}_f: (\states \times Q) \times \actions \rightarrow \mathbb{R}$ where $\states$ is the set of states of the CTMDP $\Mdp$ and $Q$ is the set of states of the \textbf{GFM} $A$.
Initially, the state space is unknown to the agent and the agent will have information only on the initial state. 
States seen are stored in a Q-table where the Q-value of the state is stored. 
The initial value of a state in the Q-table is zero.
% Variable $i$ keeps track of the number of episodes completed and variable $j$ keeps track of the length of an episode.
The number of episodes to be conducted and the length of each episode are defined by the user, let these be denoted by $k$ and $eplen$ respectively.
In each episode, the RL agent picks an action from its current state in the CTMDP according to the RL schedule and observes the next state and the time spent in the current state. 
It also picks the transition in the GFM based on the observed state in the CTMDP. 
For each transition taken, the reward obtained is based on the reward function $r'$.
% \todo{Where is $r'$ defined?}
% The variables $s$ and $q$ represent the current state of $\Mdp$ and $A$ respectively and are initialised to the respective initial states. 
% The action $a$ from the current state of CTMDP and the transition $t$ from the state of GFM are picked according to the RL policy used (eg. $\epsilon$-greedy). 
% The RL agent observes the next state $(s',q')$ and the time spent $\tau$ in the current state. 
% The reward function $rew$ is defined based on $r'$ defined previously. 
% For a given state $(s,q)$ and action $a$, if $\tau$ is the observed time spent in $(s,q)$ then,
% $$
%  rew((s,q),a,\tau) = \begin{cases}
%  \tau & \text{if ($s,q$) is an accepting state}\\
%  0 & \text{otherwise}
%  \end{cases}
%  $$
% Variable $r$ stores the reward obtained in each iteration of the episode. 
The Q-function is updated according to the Q-learning rule defined in Section~\ref{prelims}. 
An episode ends when the length of the episode reaches $eplen$. 
After the completion of $k$ episodes, we obtain a schedule $\sigma$ by choosing the action that gives the highest Q-value from each state.
The schedule learnt by the Q-learning algorithm converges to an optimal schedule as the number of training episodes tend to infinity.
We provide a pseudocode of the algorithm 
% for the expectation semantics 
in Appendix~\ref{algo}. 
% \begin{algorithm}
% \caption{Algorithm for expectation semantics}\label{algo:expt}
% \hspace*{\algorithmicindent} \textbf{Input:}  \text{Initial state $s_{init}$, SLDBW $A$, discount factor $\gamma$,}\\
% \hspace*{\algorithmicindent} \hspace{11mm}\text{reward function $R'$, number of episodes $k$,} \\
% \hspace*{\algorithmicindent}\hspace{11mm} \text{learning rate $\alpha$, episode length $eplen$}\\
% \hspace*{\algorithmicindent} \textbf{Output:}  \text{Optimal strategy $\sigma$}
% \begin{algorithmic}[1]
% \STATE Initialise $Q_f$ to all zeroes
% \STATE $i \leftarrow 0$
% \WHILE{$i < k$}
% \STATE $s \leftarrow \initstate$
% \STATE $q \leftarrow q_0$
% \STATE $r \leftarrow 0$
% \STATE $j \leftarrow 0$
% \WHILE{$j < eplen$}
% \STATE Choose action $a$ according to the RL policy
% \STATE Take action $a$, observe next state $s'$, and time $\tau$
% \STATE Choose transition $t$ in SLDBW according to the RL policy
% \STATE Take transition $t$ in SLDBW, observe next state $q'$
% \STATE $r \leftarrow R'(s,q,a,\tau)$
% \STATE $Q_f(s,q,a) \leftarrow Q_f(s,q,a) +\alpha \bigl[ r + e^{-\gamma\tau} \max_{a' \in \actions} Q_f(s',q',a')- Q_f(s,q,a) \bigr]$
% \STATE $s \leftarrow s'$
% \STATE $q \leftarrow q'$
% \ENDWHILE
% \STATE $i \leftarrow i+1$
% \ENDWHILE
% \end{algorithmic}
% \end{algorithm}
% \subsection{Blackwell Optimality In CTMDP}
% \input{blackwell_optimality}
% \input{table}
% \input{csl}

% \section{Reinforcement Learning of CTMDP}
% \label{sec:RL_methods}
% \input{RL_methods}

\section{Experimental Evaluation}
\label{sec:expt}

% \begin{table*}[t!]
%   \caption{Q-learning results.  The default values of the learner
%     hyperparameters are: $\zeta = 0.99$, $\epsilon=0.1$, $\alpha=0.1$,
%     tol$=0.01$, ep-l$=30$, and ep-n$=20000$.  Times are in seconds.}
%   \label{tab:experiment}
%   \centering
%   %\small
%   \begin{tabular}[c]{l|cccccccccccccc}
% Name & states & prod. & Prob. & Est. Prob. & Time 1 & Est-Avg & Time 2\\\hline
% \texttt{DynamicPM-tt\_3\_qs\_2}  & 816 & 825 & 1 & 1    & 2.73944 & 1.95225 & 2.09464\\
% \texttt{ErlangStages-k500\_r10}  & 508 & 509 & 1 & 1    & 2.73944 & 5.786 & 1.12393 \\
% \texttt{ErlangStages-k2000\_r10}  & 2008 & 2009 & 1 & 1    & 2.73944 & 5.78565 &1.22119 \\
% \texttt{PollingSystem-jt1\_qs1}  & 16 & 20 & 1     & 1    & 78.00 &18.8849 & 1.12183\\
% \texttt{PollingSystem-jt1\_qs4}  & 348 & 352 & 1     & 1    & 78.00 & 0.3302 &1.3917\\
% \texttt{PollingSystem-jt1\_qs7}  & 1002 & 1006 & 1     & 1    & 1.72974 & 0.0004 &1.6969 \\
% \texttt{QS-lqs\_1\_rqs\_1\_jt\_2}  & 266 & 282 & 1     & 1    & 1.72974 & 2.38395 &1.28006 \\
%  \texttt{QS-lqs\_1\_rqs\_1\_jt\_5}  & 3977 & 4152 & 1     & 1    &  1.98968 & 0.3981 & 8.51865\\
% \texttt{QS-lqs\_2\_rqs\_2\_jt\_3}  & 11045 & 24672 & 1     & 1 & & 0.49755 & 29.527 \\
% \texttt{SJS-procn\_2\_jobn\_2}  & 17 & 21 & 1     & 1    & 4.05513 & 26.4677 &1.09653\\
% \texttt{SJS-procn\_2\_jobn\_6}  & 7393 & 7405 & 1     & 1    & 6.39278 &9.1482&3.65823\\
% \texttt{ftwc\_001\_mrmc}  & 82 & 122 & 0.999774 & 0.999779    &  19.8628 &0.00075 &19.4509
%   \end{tabular}
% \end{table*}
% \input{table}
% {\color{blue}{Todo: Ashutosh to add info about the tool.}}
We implemented the reward schemes described in the previous sections in a C++-based tool \textsc{Mungojerrie}~\cite{hahn2021mungojerrie} 
which reads CTMDPs described in the PRISM language \cite{kwiatk11} and $\omega$-regular automata written in the \emph{Hanoi Omega Automata} format \cite{Babiak15}. 
Our implementation provides an Openai-gym~\cite{Brockm16} style interface for RL algorithms and supports probabilistic model checking for CTMDPs based on uniformization. 

Table~\ref{tab:experiment} shows the results of the evaluation of our algorithms on a set of CTMDP benchmarks from the Quantitative Verification Benchmark set ({\tt https://qcomp.org}).
% A brief description of these benchmarks is given below. 
\texttt{RiskReward} is based on Example~\ref{example:1} with $\lambda(0,b) = 10$ and $r = 9$.
\texttt{DynamicPM-tt\_3\_qs\_2}  models encode dynamic power management problem based on~\cite{DPM00}.
Queuing System (QS) models \texttt{QS-lqs\_i\_rqs\_j\_jt\_k} are based on a CTMDP modelling of queuing systems with arrival rate $i$, service rate $j$, and jump rate $k$ as the key parameters.
\texttt{ftwc\_001\_mrmc} models consist of two networks of $n$ workstations each where each network is interconnected by a switch communicating via a backbone. The components may fail arbitrarily, but can only be repaired one at a time. The initial state is the one where all components are functioning, and the goal state is the one where in both networks either all the workstations or all the switches are broken.
The Polling System examples \texttt{PollingSystem-jt1\_qsj} consist of $j$ stations and $1$ server. Here, the incoming requests of $j$ types are buffered in queues of size $k$ each, until they are processed by the server and delivered to their station. The system starts in a state with all the queues being nearly full. We consider 2 goal conditions: (i) all the queues are empty and (ii) one of the queues is empty.
Finally, the stochastic job scheduling (SJS) examples \texttt{SJS-procn\_i\_jobn\_j} model multiple processors ($i$) with a sequence of independent jobs ($j$) with a goal job completion.

The results are summarized in Table~\ref{tab:experiment}.  
For each model, we provide the number of states in the CTMDP ({\bf states}) and in the product CTMDP ({\bf prod}), the probability of satisfaction ({\bf Sat. Prob.})) of the objective for the satisfaction semantics, estimated probability for the satisfaction semantics ({\bf Est. Sat.}) by the RL algorithm, and time ({\bf Time 1}) spent in learning that schedule. 
The probability of satisfaction for the expectation semantics ({\bf Exp. Prob.}), estimated probability by the RL algorithm ({\bf Est. Exp.}), and the learning time ({\bf Time 2}) for the expectation semantics are provided next. All of our timings and values are averaged over three runs with randomly chosen seeds.
We kept the default values for the hyperparameters as shown in 
Table~\ref{tab:experiment}.

% The results are summarized in Table~\ref{tab:experiment}.  
% For each model, we provide the number of states in the CTMDP ({\bf states}) and in the product CTMDP ({\bf prod}), the probability of satisfaction ({\bf Sat. Prob.})) of the objective for the satisfaction semantics, estimated probability ({\bf Est. Prob.}) by the RL algorithm, and time ({\bf Time 1}) spent in learning that schedule. 
% The estimated expected average ({\bf Est. Avg.}) and the learning time ({\bf Time 2}) for the expectation semantics are provided next. All of our timings and values are averaged over three runs with randomly chosen seeds.
% We kept the default values for the hyperparameters as shown in 
% Table~\ref{tab:experiment}.
% It should be noted that the models used for testing are designed for checking reachability property and therefore the probability of satisfaction for both the semantics are the same here. The first benchmark in the table (riskReward) represents Example~\ref{example:1} where $\lambda(0,b) = 10$ and $r = 9$. It can be seen that the probability of satisfaction of the objective for satisfaction semantics gives $1$ (picking action $a$) while for expectation semantics, the probability of satisfaction is $0.9$ (picking action $b$).

Our experimental results demonstrate that the proposed RL algorithms are effective in handling medium sized CTMDPs. 
Since for the expectation semantics, the optimal probability was computed using linear programming, we can notice that the RL algorithm efficiently estimates the optimal probability and computes the optimal schedule.

% Results 

% {\color{blue} Perhaps add a discussion talking about similarities with CSL here.}

% {\color{red} Need more related work, including the following:
% \begin{itemize}
%     \item  RL in real-time context
%     \item work on CTMDPs
%     \item Work on omega-regular semantics
%     \item Work on logics like LTL, DC, STL, MTL, and CSL
%     \item Other models for timed systems like timed automata..
%     \item 
%     Summary of work on omega-regular RL
%     \item 
%     semi-markov decision process reinfrocement learning 
%     \item 
%     Timed automata, probabilsitic timed automata, stochastic timed automata 
% \end{itemize}
% }

\section{Conclusion}
\label{sec:conc}
Continuous-time MDPs are canonical models to express nondeterministic and stochastic behavior under dense-time semantics.
Reinforcement learning (RL) provides a sampling-based method to compute an optimal schedule in the absence of an explicit environment model. 
The RL approach for CTMDPs has recently received considerable attention~\cite{GZ16,RS13}. 
Our work enabled the specification of learning objectives in CTMDPs as $\omega$-regular specifications.
To accommodate temporal modelling, we consider two semantics of $\omega$-regular specifications (that include LTL objectives) and provide translations to scalar reward forms amenable for model-free reinforcement learning.
We believe that this work will open doors to study and develop model-free reinforcement learning for continuous-time models that go beyond CTMDPs and allow temporal constraints on planner's choices and residence-time requirements. 

% A natural next step is to investigate the CTMDP-like environments where the controller can select timed actions in the environment. These models will combine the timed automata~\cite{AD94} based modelling with clock variables to continuous distributions of CTMDPs.
% Another important direction is to improve the scalability of RL-driven scheduler synthesis to develop deep RL algorithms capable of exploiting symmetry inherent in dense-time environments.

\section*{Acknowledgement}
This work is partially supported by DST-SERB grant SRG/2021/000466 and by the National Science Foundation (NSF) grant CCF-2009022 and by NSF CAREER award CCF-2146563.

\bibliography{references}
\pagebreak

\onecolumn
\appendix
\begin{center}
    {\LARGE Appendix}
\end{center}
\section{Example of end-components} \label{app:EC}
We show an example of a CTMDP in Figure~\ref{fig:example} which has three end-components $\widehat{A}, \widehat{B}, \text{ and } \widehat{C}$.

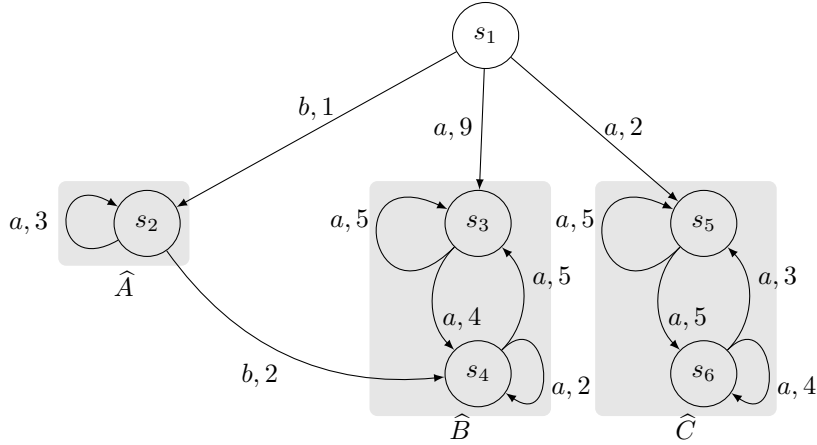
\begin{figure}[h]
	\centering
	{
		\scalebox{1}{
			\begin{tikzpicture}[auto]
				\node[state] (s1) at (-3,0.5) {$s_1$};
				\node[state] (s6) at (-0.1,-4) {$s_6$};
				\node[state] (s5) at (-0.1,-2) {$s_5$};
				\node[state] (s3) at (-3.1,-2) {$s_3$};
				\node[state] (s4) at (-3.1,-4) {$s_4$};
				\node[state] (s2) at (-7.5,-2) {$s_2$};
				% \coordinate[at= ($(state_0) + (45:0.5)$),label={0:a,$0$}] (state_0_a);
				% \draw[-] (s0) -- (state_0_a);
				% \node[] at (-2.3,0.3) {$a$};
				% \node[] at ($(s4) + (0.8,0.5)$) {$a$};
				% \node[] at ($(s6) + (0.8,0.5)$) {$a$};
				% \node[] at ($(s3) + (-0.9,-0.7)$) {$a$};
				% \node[] at ($(s5) + (-0.9,-0.7)$) {$a$};
			
				\path[->]
					(s1) edge node[right]{$a,2$} (s5)
					(s1) edge node[left]{$a,9$} (s3)
					(s1) edge node[above] {$b,1$} (s2)
					(s5) edge[out=225,in=135] node[below right]{$a,5$} (s6)
					(s5) edge[out=225,in=155,looseness=8] coordinate (e5loop) node[above left]{$a,5$} (s5)
					(s6) edge[out=45,in=-35,looseness=4] coordinate (e6loop) node[below right] {$a,4$} (s6)
					(s6) edge[out=45,in=-45] node[above right] {$a,3$} (s5)
					(s3) edge[out=225,in=135] node[below right]{$a,4$} (s4)
					(s3) edge[out=225,in=155,looseness=8] coordinate (e3loop) node[above left]{$a,5$} (s3)
					(s4) edge[out=45,in=-35,looseness=4] coordinate (e4loop) node[below right] {$a,2$} (s4)
					(s4) edge[out=45,in=-45] node[above right] {$a,5$} (s3)
					(s2) edge[loop left] coordinate (e2loop) node[left, inner sep=7pt]{$a,3$} (s2)
					(s2) edge[out=-80,in=180,bend right] node[below left] {$b,2$} (s4);
				
				\node[rectangle,rounded corners=3pt,draw=none,fill=black,fill opacity=0.1,fit=(s2) (e2loop)] (rectA) {};
				\node[rectangle,rounded corners=3pt,draw=none,fill=black,fill opacity=0.1,fit=(s3) (s4) (e3loop) (e4loop)] (rectB) {};
				\node[rectangle,rounded corners=3pt,draw=none,fill=black,fill opacity=0.1,fit=(s5) (s6) (e5loop) (e6loop)] (rectC) {};
				\node[] at ($(rectA) + (0,-0.8)$) {$\widehat{A}$};
				\node[] at ($(rectB) + (0,-1.7)$) {$\widehat{B}$};
				\node[] at ($(rectC) + (0,-1.7)$) {$\widehat{C}$};
			\end{tikzpicture}
		}
	}
	\caption{A CTMDP with three end-components.}\label{fig:example}
\end{figure}

\section{Example of Uniformization}
We give an example of a non-uniform CTMDP and its uniformized version.
\label{uniform}
\begin{figure}[h]
    \centering
    \begin{subfigure}[b]{0.48\linewidth}
\begin{tikzpicture}[shorten >=1pt, node distance=3 cm, on grid, auto,thick,initial text=]
\begin{scope}
\node (l0) [state,fill=safecellcolor]  {$q_0$};
\node (l1) [state, fill=safecellcolor, right = of l0,xshift = -0.75cm]   {$q_1$};
\end{scope}
 \begin{scope}
\path [->]
    (l0) edge [bend left]  node [above] {$a_1, 3$}   (l1)
    (l0) edge [bend right]  node [below] {$a_2,6$}   (l1)
    (l1) edge [loop above] node [above] {$a_3,2$}   ()
    ;
\end{scope}
\end{tikzpicture}
\caption{Non-uniform CTMDP where the exit-rates are different for various state action pairs. } \label{fig:PM1}
\end{subfigure}
\hspace{0.5em}
\begin{subfigure}[b]{0.45\linewidth}
\begin{tikzpicture}[shorten >=1pt, node distance=3 cm, on grid, auto,thick,initial text=]
\begin{scope}
\node (l1) [state, fill=safecellcolor]  {$q_1$};
\node (l0) [state, fill=safecellcolor, left = of l1,xshift = 0.75cm]  {$q_0$};
\end{scope}
\begin{scope}
\path [->]
     (l0) edge [loop above] node [above] {$a_1,3$}
    (l0) edge [bend left]  node [above] {$a_1,3$}   (l1)
    (l0) edge [bend right]  node [below] {$a_2,6$}   (l1)
    (l1) edge [loop above] node [above] {$a_3,6$}   ()
    ;
\end{scope}
\end{tikzpicture}
\caption{A Uniform CTMDP where the exit-rate for every state-action pair is $6$. } \label{fig:PM2}
\end{subfigure}
\caption{Uniformization of a CTMDP}
\label{fig:P1}
\end{figure}
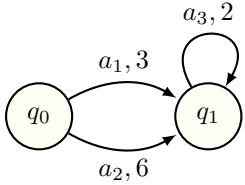
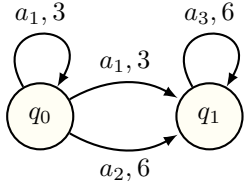

\section{Good-for-CTMDP B\"uchi Automata}
\label{app:gfm}
In this section, we provide a formal definition of good-for-CTMDP automata.
\paragraph{Product CTMDP.}
 Given a \emph{labelled} CTMDP $(\Mdp, \atomicprop, \labelling)$ where $\atomicprop$ is a set of atomic propositions, and $\labelling: \states \rightarrow 2^\atomicprop$ is a labelling function and a   B{\"u}chi automaton $\oautomata = (2^{\atomicprop}, \ostate, \oinitstate, \otrans, \acceptingc)$, the product CTMDP is defined as $\Mdp \times \oautomata = ((\states \times \ostate),(\initstate,\oinitstate),\actions,\transR^{\times},\acceptingc^{\times})$ where the rates are $\transR^{\times} : (\states \times \ostate) \times \actions \times (\states \times \ostate) \rightarrow \Reals_{\geq 0}$ such that $\transR^{\times} ((s,q),a,(s',q') = \transR(s,a,s')$ if $\transR(s,a,s')>0$ and $\otrans (q,\labelling(s)) = \{q'\}$.
If $F$ is the set of accepting 
% transitions 
states
in $\oautomata$, then the accepting condition is a set 
% transitions $F^\times$ where $((s,q),a,(s',q')) \in F^\times$ iff $(q,\labelling(s),q') \in F$ 
$F^\times$ of states where $(s,q) \in F^\times$ iff $q \in F$. An example of a product CTMDP is given in Figure~\ref{fig:grid-world}.  

Given an MDP $\Mdp$, a B\"uchi automaton  $\oautomata$, and product $\Mdp \times \oautomata$, we define the following two  problems:
\begin{enumerate}
    \item {\bf Satisfaction Semantics.} 
    Compute a schedule of $\Mdp$ that maximizes the probability of visiting accepting states $F$ of $\oautomata$ infinitely often.
    We define the satisfaction probability $\PSat^{\Mdp \times \oautomata}(s, \sigma)$ of a schedule $\sigma$ from starting state $s$ as: 
\begin{equation*}
\Pr{}^{\Mdp \times \oautomata}_{\sigma}(s) \set{\forall_i  \exists_{j{\geq} i} [X_j \in F^\times] }.
\end{equation*}
The optimal satisfaction probability
$\PSemSat^{\Mdp \times \oautomata}_{\oautomata}(s)$ for specification $\oautomata$ 
is defined as $\sup_{\sigma \in \Sigma_{\Mdp \times \oautomata}} \Pr^{\Mdp \times \oautomata}_{\sigma}(s, \sigma)$ and we say
that $\sigma$ is an optimal schedule for $\oautomata$ if
$\PSemSat^{\Mdp \times \oautomata}_{\oautomata}(s, \sigma) (s) = \PSemSat^{\Mdp \times \oautomata}_{\oautomata}$.

\item {\bf Expectation Semantics.} Compute a schedule of $\Mdp$ that maximize the long-run expected average time spent in the accepting states of $\oautomata$.
We define the expected satisfaction time $\ESat^{\Mdp \times \oautomata}_{\oautomata}(s, \sigma)$ 
of $\sigma$ from starting state $s$ as: 
\begin{equation*}
 \mathbb{E}^{\Mdp \times \oautomata}_{\sigma}(s) \set{ \liminf_{n \rightarrow \infty} \frac{\sum_{i=1}^{n} [X_i {\in} F] D_i}{T_n}}.
\end{equation*}
The optimal expected satisfaction time
$\ESat^{\Mdp \times \oautomata}_{\oautomata}(s)$ for $\oautomata$ is defined as $\sup_{\sigma \in \Sigma_{\Mdp}} \ESat^{\Mdp \times \oautomata}_{\oautomata}(s, \sigma)$ and we say that $\sigma \in \Sigma_\Mdp$ is an optimal expected-satisfaction schedule for $\oautomata$ if
$\ESat^{\Mdp \times \oautomata}_{\oautomata}(s, \sigma) = \ESat^{\Mdp \times \oautomata}_{\oautomata} (s)$.
\end{enumerate}

We call a B{\"u}chi automaton $\oautomata = (2^{\atomicprop}, \ostate, \oinitstate, \otrans, \acceptingc)$ good-for-CTMDP if for every \emph{labelled} CTMDP $(\Mdp, \atomicprop, \labelling)$ where $\atomicprop$ is a set of atomic propositions, we have that 
\begin{eqnarray*}
    \PSemSat^{\Mdp}_{\oautomata}(s, \sigma) &=&\PSat^{\Mdp \times \oautomata}_{\oautomata}(s, \sigma)~\text{ and }\\
    \ESemSat{}^{\Mdp}_{\oautomata}(s, \sigma) &=&
    \ESat^{\Mdp \times \oautomata}_{\oautomata}(s, \sigma).
\end{eqnarray*}
A good-for-CTMDP automaton allows the computation of the optimal schedule by solving the corresponding problem on the product CTMDP. 
If a B\"uchi automaton is good-for-MDP, then one can show via uniformization that it is also good-for-CTMDPs.
There exists several syntactic characterizations of good-for-MDP automata including suitable limit-deterministic B\"uchi automata (SLDBA)~\cite{sickert2016limit} and slim automata~\cite{HPSS20}.
Moreover, every LTL specification can be effectively converted into a GFM B\"uchi automata.

% \begin{theorem}[Good-for-MDP B\"uchi Automata]
% \end{theorem}
\section{Need For Memory for Optimal Schedules}
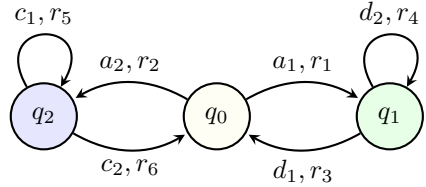
\label{exmp:memory}
\begin{figure}[h]
    \centering
     \begin{tikzpicture}[shorten >=1pt, node distance=2.3 cm, on grid, auto,thick,initial text=]
\begin{scope}[every node/.style={scale=1}]
\node (l0) [state, fill = safecellcolor] {$q_0$};
\node (l1) [state,right = of l0, fill = goodcellcolor]   {$q_1$};
\node (l2) [state,left = of l0,fill = badcellcolor]   {$q_2$};
\end{scope}
 \begin{scope}
\path [-stealth, thick]
    (l0) edge [bend left]  node [above] {$a_1,r_1$}   (l1)
    (l0) edge [bend right]   node [above] {$a_2,r_2$}   (l2)
    (l1) edge [bend left]  node [below] {$d_1,r_3$}   (l0)
    (l1) edge  [loop above] node [above] {$d_2,r_4$}   ()
    (l2) edge  [loop above] node [above] {$c_1,r_5$}   ()
    (l2) edge [bend right] node [below] {$c_2, r_6$} (l0)
    ;
\end{scope}
\end{tikzpicture}
  \label{fig:exmp2}
  \end{figure}
  
\begin{example}[Why memory is required to satisfy $\omega$-regular properties.]
\label{ex:mem}
Consider the CTMDP given above, let the atomic propositions be $\mathtt{b}$ and $\mathtt{g}$ representing blue and green respectively.
Thus the labels on states $q_1$ and $q_2$ are defined as, $\labelling(q_1) = \neg \mathtt{b} \land \mathtt{g}$ and $\labelling(q_2) = \mathtt{b} \land \neg \mathtt{g}$.   
Consider the $\omega$-regular property to be satisfied be $\phi = \always \eventually(\mathtt{b}) \land \always \eventually(\mathtt{g})$. 

We can observe that for a schedule to satisfy this property, both $q_1$ and $q_2$ have to be seen infinitely often. 
As there are no transitions between $q_1$ and $q_2$, both the states can be visited infinitely often only via $q_0$. 
Therefore, the schedule cannot be memoryless as choosing any one action from $q_0$ would not satisfy $\phi$.

\end{example}

% \begin{example}[$\omega$-regular objectives and memory]
% \label{ex:mem}
% This example demonstrate the need of memoryful schedules for $\omega$-regular objectives. Consider an MDP with single state and two self-loops with actions $a$ and $b$. 
% For an specification $\nextt a \wedge \nextt \nextt b$, it is clear that we need memory to satisfy the objective.
% \end{example} 

\section{Blackwell Optimality In CTMDP}
\label{sec:blackwell_optimality}
% \subsection{Blackwell Optimality In CTMDP}
% \track{Though both of our semantics use different reward machines, for learning schedules for both objectives, we reduce the problem to maximising the expected average reward.
% We use off-the-shelf RL algorithms for learning the schedulers in both settings.}
% Standard RL algorithms try to optimise discounted payoff objective while 
% % as shown above, for the expectation semantics, 
% we need to optimise the expected average reward. Blackwell optimality allows us to use a schedule that optimises the expected discounted payoff with a high discount factor and such a schedule also optimises the expected average payoff.

% \vspace{0.5em}\noindent\textbf{Existence of Blackwell Optimal Schedules in CTMDP.} 
% We give a simple uniformization based proof on the existence of a Blackwell optimal pure schedule in a CTMDP.
In this section, we provide a uniformization based proof of Theorem~\ref{corollary:1}.
Consider a CTMDP $\Mdp$, a pure schedule $\sigma$, and a continuous-time discounting with parameter $\dfactor > 0$. 
% The one-step expected reward obtained from taking action $a \in \actions{}(s)$ from state $s \in \states$ is given by $\rho(s,a) = rew(s,a) + \frac{rew(s)}{\dfactor + \lambda(s,a)}$ (~$\!\!$\cite{puterman2014markov}~Eq 11.5.3). 
We define a function $\ctmdprate : \states \times \actions \rightarrow [0,1) $ where $\ctmdprate(s,a) = \frac{\lambda(s,a)}{\lambda(s,a) + \dfactor}$ where $\dfactor > 0$.
We call $\ctmdprate(s,a)$ the \emph{discount rate} of the state-action pair $(s,a)$ in $\Mdp$.
So, the expected discounted reward (also known as value of $s$) $\discobjective^{\Mdp[\sigma]} (\alpha)(s)$ is given by
% For a pure schedule $\sigma$, the value of a state $s$, denoted by $\valuesigma_{\ctmdprate(s,\sigma(s))}$,
% \track{CHANGE NOTATION} 
% is given by 
\begin{equation}\label{eq:7.1}
\begin{split}
     \rho(s,a_{\sigma}) + 
     \ctmdprate(s,a_{\sigma})
     \sum_{s' \in s} \pmtrx(s,a_\sigma,s') 
     \discobjective_{\sigma}^{\Mdp}(\alpha)(s')
\end{split}
\end{equation}
Consider a DTMDP $\mathcal{N}$, a schedule $\sigma$, and a discount rate $0\leq \drate < 1$. Let $\valuesigma_{\drate}(\mathcal{N},s)$ denote the total discounted value from state $s$ in $\mathcal{N}$ under schedule $\sigma$.

A pure schedule $\bschedule$ is \emph{Blackwell optimal} in $\mathcal{N}$ if there exists a threshold discount rate $0 \leq \thrate < 1$ such that for any discount rate $\thrate \leq \drate < 1$, we have $\valuestar_{\drate}(\mathcal{N},s) \geq \valuesigma_{\drate}(\mathcal{N},s)$ for all $\sigma \in \Sigma_{\mathcal{N}}$. It is known that a Blackwell optimal schedule maximises both discounted and average reward objectives in DTMDPs.
% From \cite{puterman2014markov}~Thm 10.1.4, we have,
% \begin{theorem}\label{puterman1}
From \cite{puterman2014markov}~(Thm 10.1.4), we have that for every DTMDP $\mathcal{N}$, there exists a Blackwell optimal \emph{pure} schedule $\bschedule$, and $\bschedule$ also maximises the average reward in $\mathcal{N}$.
% \end{theorem}
Now, given a CTMDP $\Mdp$, let $C$ be a constant such that $C \geq \lambda(s,a)$ for all state-action pairs in $\Mdp$. 
Let $\uMdp$ be the uniformized CTMDP of $\Mdp$ with constant exit rate $C$, and let $\upmtrx$ be the probability matrix of $\uMdp$. As the exit rate $\lambda(s,a) = C$ for all $s\in \states$ and $a \in \actions{}(s)$, we have that $\ctmdprate(s,a) = \frac{C}{C+\dfactor}$ for all state-action pairs. We denote this discount rate by $\udrate$. 
The value of a state $s$ under a schedule $\sigma$ in $\uMdp$, 
% \track{under schedule $\sigma$}
denoted $\discobjective^{\uMdp^{\sigma}}(\alpha)(s)$
% $\valuesigma_{\udrate}(\uMdp,s)$ 
is given by, 
\begin{equation}\label{eq:7.2}
     \Bar{\rho}(s,a_\sigma) + \udrate \sum_{s' \in s} \upmtrx(s,a_\sigma,s') \discobjective^{\uMdp^{\sigma}}(\alpha)(s')
\end{equation}
where $\Bar{\rho}(s,a) = \rho(s,a)\cdot \frac{\dfactor + \lambda(s,a)}{\dfactor + C} $. 
We extend the above result of existence of Blackwell optimal schedules in DTMDPs to uniform CTMDPs.
\begin{lemma} \label{lemma:1}
For a uniform CTMDP $\uMdp$, there exists a Blackwell optimal schedule $\bschedule$. Further, $\bschedule$ also maximises the expected average reward in $\uMdp$.
\end{lemma}
\begin{proof}
Consider a DTMDP $\mathcal{N}$ with the same set of states as that of $\uMdp$, one step reward function $\Bar{r}$ and probability matrix $P_{\mathcal{N}} = \upmtrx$. For a pure schedule $\sigma$ and a discount rate $0 \leq \drate < 1$, the value of a state $s$ in $\mathcal{N}$ is 
\begin{equation}\label{eq:7.3}
    \valuesigma_{\drate}(\mathcal{N},s) = \Bar{r}(s,\sigma(s)) + \drate \sum_{s' \in s} \upmtrx(s,\sigma(s),s') \valuesigma_{\drate}(\mathcal{N},s')
\end{equation}
We observe that equation \ref{eq:7.3} is identical to equation \ref{eq:7.2} when $\udrate = \drate$. Therefore, the set of equations defining the values of states in $\mathcal{N}$ and $\uMdp$ are identical. Let this set be denoted by $E^\sigma$. From \cite{puterman2014markov}~Thm 6.1.1, we know that for each stationary schedule $\sigma$, there exists a unique solution for $E^\sigma$. The set of pure schedules in $\uMdp$ and $\mathcal{N}$ are equal as the set $S$ of states and the set $\av$ of available actions from each state are the same for both.\\*
Therefore, for a pure schedule $\sigma$, and discount rates $0 \leq \drate = \udrate < 1$, we have
\begin{equation}\label{eq:7.4}
    \valuesigma_{\udrate}(\uMdp,s) = \valuesigma_{\drate}(\mathcal{N},s) \text{\: for $\udrate = \drate$, for all states $s$}
\end{equation}
From Theorem~10.1.4 in \cite{puterman2014markov}, we know that in $\mathcal{N}$, there exist a Blackwell optimal pure schedule $\sigma^*$, and a threshold discount rate $\thrate$ such that 
\begin{align}\label{eq:7.5}
 \valuestar_{\drate}(\mathcal{N},s) \geq \valuesigma_{\drate}(\mathcal{N},s)   
\end{align}
for all $\sigma \in \Sigma_{N}$ and $\thrate \leq \drate < 1$.\\
From equations \ref{eq:7.4} and \ref{eq:7.5} we can conclude that 
\begin{equation}\label{eq:7.60}
    \valuestar_{\udrate}(\uMdp,s) \geq \valuesigma_{\udrate}(\uMdp,s)   
\end{equation}
for all $\sigma \in \Sigma_{\uMdp}^{pure}$ and $\thrate \leq \udrate < 1$.\\
From \cite{puterman2014markov}~Thm 11.5.2(d), we know that there exists an optimal pure schedule maximising the discounted reward in a CTMDP. Therefore,
\begin{equation}\label{eq:7.}
    \valuestar_{\udrate}(\uMdp,s) \geq \valuesigma_{\udrate}(\uMdp,s)   
\end{equation}
for all $\sigma \in \Sigma_{\uMdp}$ and $\thrate \leq \udrate < 1$.\\
A similar argument can be made to show that $\sigma^*$ also maximises the expected average reward in $\uMdp$.
\end{proof}

The above lemma proves the existence of a Blackwell optimal schedule in uniform CTMDPs. 
We further extend this result to general CTMDPs which is the main result of this section.

\begin{lemma}\label{lemma:2}
If $\sigma^*$ is a Blackwell optimal pure schedule in $\uMdp$, then it is also Blackwell optimal in $\Mdp$.
\end{lemma}
\begin{proof}
Since $\bstrategy$ is a Blackwell optimal pure schedule in $\uMdp$, there exists a threshold discount rate $\uthrate$ such that for all $\uthrate \leq \udrate < 1$, we have that 
\begin{equation}\label{eq:7.7}
    \valuestar_{\udrate}(\uMdp,s) \geq \valuesigma_{\udrate}(\uMdp,s)   
\end{equation}
for all $\sigma \in \Sigma_{\uMdp}$.\\
The set of pure schedules in $\Mdp$ and $\uMdp$ are the same. From \cite{puterman2014markov}~Thm 11.5.2(d), we know that there exists an optimal pure schedule maximising the discounted reward in $\Mdp$. \cite{puterman2014markov}~Prop 11.5.1, states that for every pure schedule $\sigma$ and a state $s$, we have that  
\begin{equation}\label{eq:7.8}
    \valuesigma_{\ctmdprate(s,\sigma(s))}(\Mdp,s) = \valuesigma_{\udrate}(\uMdp,s)   
\end{equation}
If $\uthrate = \frac{C}{C+\dfactor}$ is the threshold discount rate in $\uMdp$, then the corresponding threshold discount rate for a state $s$ in $\Mdp$ is given by 
$\ctmdpthrate(s,a) = \frac{\lambda(s,a)}{\lambda(s,a)+\dfactor_{o}}$.\\
From equations \ref{eq:7.7} and \ref{eq:7.8}, we can conclude that for each state $s$ in $\Mdp$, there exist a pure schedule $\bstrategy$, and a threshold discount rate $\ctmdpthrate(s,\bstrategy(s)) $ such that for all $\ctmdpthrate(s,\bstrategy(s)) \leq \ctmdprate(s,\bstrategy(s)) < 1$, we have that
\begin{equation}\label{eq:7.9}
    \valuestar_{\ctmdprate(s,\bstrategy(s))}(\Mdp,s) \geq \valuesigma_{\ctmdprate(s,\sigma(s))}(\Mdp,s)   
\end{equation}
for all $\sigma \in \Sigma_{\Mdp}$.
As the set of states is finite in $\Mdp$, the threshold discount rate for $\Mdp$ is given by
    $\ctmdpthrate^{\Mdp} = \max_{(s,a) \in \states \times \actions} \ctmdpthrate(s,a)$.
\end{proof}
Thus, any Blackwell optimal schedule $\bstrategy$ in $\uMdp$ is also Blackwell optimal in $\Mdp$. The following lemmas show that $\sigma^*$ also maximises the expected average reward.
\begin{lemma}\label{lemma:3}
An optimal schedule maximising the expected average reward in $\uMdp$ also maximises the expected average reward in $\Mdp$.
\end{lemma}
\begin{proof}
For a pure schedule $\sigma$ , the expected average reward in $\Mdp$ is denoted by $\avgrew^{\sigma}(\Mdp)$. From \cite{puterman2014markov}~Chap 11.5.3, we observe that for a pure schedule $\sigma$, 
\begin{equation}\label{eq:7.10}
    \avgrew^{\sigma}(\Mdp) = \avgrew^{\sigma}(\uMdp)\cdot C 
\end{equation}
Let $\sigma'$ be a pure schedule maximising the expected average reward in $\uMdp$, i.e, 
\begin{equation}\label{eq:7.11}
 \avgrew^{\sigma'}(\uMdp) = \sup_{\sigma \in \Sigma_{\uMdp}} \avgrew^{\sigma}(\uMdp)
\end{equation}
The set of pure schedules in $\Mdp$ and $\uMdp$ are equal and we know that there exists a pure schedule maximising the average reward in a CTMDP (\cite{puterman2014markov}~Thm 11.4.6(d)). Therefore, from equations \ref{eq:7.10} and \ref{eq:7.11} we can conclude that, 
\begin{equation}\label{eq:7.12}
    \avgrew^{\sigma'}(\Mdp) = \sup_{\sigma \in \Sigma_{\Mdp}} \avgrew^{\sigma}(\Mdp)
\end{equation}
Therefore, $\sigma'$ is an optimal schedule maximising the average reward in $\Mdp$.
\end{proof}
\begin{lemma}\label{lemma:4}
A Blackwell optimal schedule $\bstrategy$ in $\Mdp$ also maximises the average reward in $\Mdp$.
\end{lemma}
\begin{proof}
From Lemma \ref{lemma:1}, we know that $\bstrategy$ is an optimal schedule maximising the expected average reward in $\uMdp$. Lemma \ref{lemma:3} shows that if $\bstrategy$ is an optimal schedule maximising the expected average reward in $\uMdp$ then $\bstrategy$ also maximises the expected average reward in $\Mdp$. Therefore, we can conclude that $\bstrategy$ is an optimal schedule maximising the expected average reward in $\Mdp$.
\end{proof}
Lemma~\ref{lemma:2} and Lemma~\ref{lemma:4} gives us the following.
\paragraph{Theorem~\ref{corollary:1}.}For a CTMDP $\Mdp$, there exists a Blackwell optimal pure schedule $\bschedule$ and a threshold $0 \leq \ctmdpthrate^{\Mdp} < 1$ such that :
\begin{inparaenum}[(1).]
    \item For any discount-rate function $\ctmdprate$ where $\ctmdprate(s,a) \geq \ctmdpthrate^{\Mdp}$ for all valid state-action pairs $(s,a)$, the schedule $\bschedule$ is an optimal schedule maximising the expected discounted reward.
    \item The schedule $\bschedule$ also maximises the expected average reward.
\end{inparaenum}
\section{Proofs from Section ~\ref{sec:theorems&algo}}
\label{app:sat}
In this section, we give a detailed proof of Theorem~\ref{theorem:4.1}.
\paragraph{Theorem~\ref{theorem:4.1}.}There exists a threshold $\zeta' \in (0,1)$ such that for all $\zeta > \zeta'$, and for every state $s$, a schedule maximising the expected average reward in $t$ is 
\begin{inparaenum}[(1)]
\item an optimal schedule in the product CTMDP $\mathcal{M} \times \mathcal{A}$ from $s$ for satisfying the $\omega$-regular objective $\phi$.
Further, since $\oautomata$ is a GFM, we have that \item $\sigma$ induces an optimal schedule for the CTMDP $\mathcal{M}$ from $s$ with objective $\phi$.
\end{inparaenum}
\begin{proof}
For a given CTMDP $\Mdp$, an embedded MDP $\embeddedMdp$ of $\mathcal{M}$ is a discrete-time MDP of the form $\embeddedMdp = (\states, \initstate, \actions{}, \transP_{\Mdp})$, that is, the transition function of $\embeddedMdp$ is derived from the probability matrix of $\Mdp$.
% Given a CTMDP $\mathcal{M}$, we denote the embedded DTMDP by $\emdp$. 
As the set of states and enabled actions from each state are the same in $\Mdp$ and $\emdp$, the set of pure schedules in them are also same. 
% Let $\phi$ be an $\omega$-regular objective, the following lemma gives us the relation between an optimal schedule satisfying $\phi$ in $\pmdp$ and one for the embedded product DTMDP $\emdp \times \oautomata$.
%  

Time does not play a role in the definition of pure schedules. Therefore, the probability of reaching a state $s$ from the initial state in a CTMDP $\Mdp$ under a pure schedule $\sigma$ is dependant only on the transition function of the embedded MDP $\emdp$. 
Now recall that for $\omega$-regular objectives, given a B\"{u}chi GFM, the states of the accepting condition of the product CTMDP need to be visited infinitely often.
As there exist optimal pure schedules for reaching such states, we get the following lemma for an $\omega$-regular objective $\phi$.

\begin{lemma} \label{lem:embedded}
There exists a pure schedule that maximises the probability of satisfying $\phi$ in $\mathcal{M} \times \mathcal{A}$ which also maximises the probability of satisfying $\phi$ in the embedded product DTMDP $\mathcal{M}_\mathcal{E} \times \mathcal{A}$.
\end{lemma}
\begin{proof}
In order to prove this lemma, we consider the semantics of a CTMC.
In a CTMC, every state $s$ has an exit rate $\lambda_s$ such that after reaching state $s$, some time $t_s$ is spent in $s$ where $t_s$ is exponentially distributed with parameter $\lambda_s$, and an outgoing transition to a state $s'$ is taken according to the probability $\transP(s,s')$ of the underlying discrete Markov chain.
Note that given a state $s$, the probability of reaching $s$ from the initial state $\initstate$ of the CTMC thus solely depends on the underlying discrete Markov chain, and not on the exit rates of the states.
Now, for every pure schedule $\sigma$, the CTMC generated is $(\Mdp \times \oautomata)^{[\sigma]}$ and the underlying discrete Markov chain is $(\Mdp_\mathcal{E} \times \oautomata)^{[\sigma]}$.

Thus, the probability of reaching the accepting end-components are the same in both $(\Mdp \times \oautomata)^{[\sigma]}$ and $(\Mdp_\mathcal{E} \times \oautomata)^{[\sigma]}$.
The lemma follows since optimal pure schedules exist for a reachability objective, in particular, there exists pure schedules maximising the probability of reaching accepting end-components, and the set of pure schedules are the same in both $(\Mdp \times \oautomata)^{[\sigma]}$ and $(\Mdp_\mathcal{E} \times \oautomata)^{[\sigma]}$.
Note that time does not play a role in the definition of stationary schedules, that is, timed stationary schedules and time-abstract stationary schedules coincide.
\end{proof}
We have shown that there exists an optimal pure schedule maximising $\phi$ in both $\pmdp$ and $\Mdp_{\mathcal{E}} \times \oautomata$. With a similar argument it also follows that an optimal pure schedule maximising the probability to reach the sink $t$ in the CTMDP $\augmdp$ also maximises the probability to reach $t$ in the DTMDP $\augmdp_{\mathcal{E}}$.

Theorem 3 from \cite{HahnPSSTW19} gives us the existence of a threshold $\zeta' \in (0,1)$ and that for any $\zeta > \zeta'$, an optimal schedule maximising the probability of reaching the sink state $t$ from a state $s$ in the DTMDP $\augmdp_{\mathcal{E}}$ also maximises the probability of satisfying $\phi$ in $\Mdp_{\mathcal{E}} \times \oautomata$.

This leads to the following statement.
There exists a threshold $\zeta' \in (0,1)$ such that for all $\zeta > \zeta'$, and for every state $s$, a schedule $\sigma$ maximising the probability $p_s(\zeta)$ of reaching the sink in $\mathcal{M}_{\mathcal{E}}^{\zeta}$ is
\begin{inparaenum}[(1)]
\item an optimal schedule in the product CTMDP $\mathcal{M} \times \mathcal{A}$ from $s$ for satisfying the $\omega$-regular objective $\phi$, and
\item induces an optimal schedule for the CTMDP $\mathcal{M}$ from $s$ with objective $\phi$.
\end{inparaenum}

The reward machine defined gives a positive reward for each time unit spent in $t$. Therefore, we can conclude that any schedule that maximises the probability of reaching $t$ also maximises the expected average reward.

\end{proof}

\section{Proofs from Section~\ref{sec:d_time}}
\label{app:expt}
In this section, we provide a proof of Lemma~\ref{lemma:6.2}.
\paragraph{Lemma~\ref{lemma:6.2}.}For a product CTMDP $\Mdp \times \oautomata$ where $\oautomata$ is a \textbf{GFM} for an $\omega$-regular objective and for a schedule $\sigma$, the expected average reward obtained w.r.t. the reward function $r'$ is equal to the expected satisfaction time in $~{(\pmdp)}$ and there exist a pure schedule that maximises this.
\begin{proof}
Consider a run $r_{inf} = (s_1,t_1,a_1,s_2,t_2,a_2...)$ in $\Mdp \times \oautomata$ under a schedule $\sigma$ . The satisfaction time is defined as 
$$
    Sat^{\Mdp{\times}\oautomata}_{\oautomata}(s,\sigma) = \liminf_{n \rightarrow \infty} \frac{\sum_{i=1,s_i \in T}^{n} t_i}{\sum_{j=1}^{n}t_j}
$$.
The reward obtained in each state $s_i$ in $r_{inf}$ is $r'(s_i)\cdot t_i$ which is $t_i$ if $s_i \in T$, and 0 otherwise.
Therefore, the average reward obtained from $r_{inf}$ is 
$$  
        \avgreward(s) = \liminf_{n \rightarrow \infty} \frac{\sum_{i=1,s_i \in T}^{n} t_i}{\sum_{j=1}^{n}t_j}
    $$.
Thus, the average reward obtained is equal to the satisfaction time for every run in $\Mdp \times \oautomata$ and therefore, the expected average reward obtained w.r.t. the reward function $r'$ is equal to the expected satisfaction time.

For proving the second part of the lemma, we use the fact that there exists an optimal pure schedule that maximises the expected average reward for any rewardful CTMDP \cite{puterman2014markov}. Therefore, there exists a pure schedule that maximises the expected residence time of $T$ in $\Mdp \times \oautomata$.
\end{proof}

\section{Pseudocode of Algorithms}
\label{algo}
In this section, we provide the pseudocode of the RL algorithm for satisfaction and expectation semantics.
\subsection*{Satisfaction Semantics}
\begin{algorithm}[H]
\caption{Algorithm for satisfaction semantics}\label{algo:sat}
\hspace*{\algorithmicindent}\hspace{2mm}\textbf{Input:}  \text{Initial state $s_{0}$, GFM $A$, discount factor $\gamma$, reward function $rew$, number of episodes (ep-n) $k$, learning rate $\beta$}\\
\hspace*{\algorithmicindent}\hspace{2mm}\textbf{Output:}\text{ Schedule $\sigma$ converging to an optimal one}
\begin{algorithmic}[1]
\STATE Initialise $\Qf$ to all zeroes
% \STATE $i \leftarrow 0$
\FOR{k episodes}
\STATE Initialise $s$ and $q$ to $s_0$ and $q_0$ respectively
\STATE Initialise $r$ to 0
% \STATE $q \leftarrow q_0$
% \STATE $r \leftarrow 0$
\WHILE{r = 0}
\STATE Choose action $a$ using schedule derived from $\Qf$
\STATE Take action $a$, observe next state $s'$ and time $\tau$
\STATE Choose non-deterministic transition $t$ in $A$ using the derived schedule ($\epsilon$-greedy)
\STATE Take transition $t$ in $A$, observe next state $q'$
\STATE $r \leftarrow rew'(s,q,a)$
\STATE $V(s',q') \leftarrow \max\limits_{a' \in \actions} \Qf(s',q',a')$
\STATE $\Qf(s,q,a) \leftarrow (1-\beta) \Qf(s,q,a)  + \beta \bigl(r {+} e^{-\gamma\tau} V(s',q') \bigr)$
\STATE $s \leftarrow s'$
\STATE $q \leftarrow q'$
\ENDWHILE
% \STATE $i \leftarrow i+1$
\ENDFOR
\STATE Initialise $\sigma$
\FOR{each state $(s,q)$}
\STATE $\sigma(s,q) = \max\limits_{a \in \actions}\Qf(s,q,a)$
\ENDFOR
\end{algorithmic}
\end{algorithm}
Recall that the reward function $rew'$ is defined as:
\[
rew'((s,q),a) = \begin{cases}
1 \text{\quad with probability $1-\zeta$ if $(s,q)$ is} \\
\text{\quad \quad accepting} \\
0  \text{\quad otherwise}
\end{cases}
\]

\subsection*{Expectation Semantics}
\begin{algorithm}[H]
\caption{Algorithm for expectation semantics}\label{algo:expt}
\hspace*{\algorithmicindent}\hspace{2mm}\textbf{Input:}  \text{Initial state $s_0$, GFM $A$, discount factor $\gamma$, reward function $rew$, number of episodes (ep-n) $k$,} \\
\hspace*{\algorithmicindent}\hspace{12mm} \text{learning rate $\beta$, episode length (ep-l) $eplen$}\\
\hspace*{\algorithmicindent}\hspace{2mm}\textbf{Output:}\text{ Schedule $\sigma$ converging to an optimal one}
% \hspace*{\algorithmicindent} \textbf{Output:}  \text{Schedule $\sigma$ converging to an optimal one}
\begin{algorithmic}[1]
\STATE Initialise $\Qf$ to all zeroes
% \STATE $i \leftarrow 0$
\FOR{$k$ episodes}
\STATE Initialise $s$ and $q$ to $s_0$ and $q_0$ respectively
% \STATE $q \leftarrow q_0$
% \STATE $r \leftarrow 0$
% \STATE $j \leftarrow 0$
\FOR{each $eplen$-length episode}
\STATE Choose action $a$ using schedule derived from $\Qf$ 
\STATE Take action $a$, observe next state $s'$ and time $\tau$
\STATE Choose non-deterministic transition $t$ in $A$ using the derived schedule ($\epsilon$-greedy)
\STATE Take transition $t$ in $A$ and observe the next state $q'$
\STATE $r \leftarrow rew((s,q),a,\tau)$
\STATE $V(s',q') \leftarrow \max\limits_{a' \in \actions} \Qf(s',q',a')$
\STATE $\Qf(s,q,a)  \leftarrow {(1-\beta)} \Qf(s,q,a)  + \beta \bigl(r {+} e^{-\gamma\tau} V(s',q') \bigr)$
\STATE $s \leftarrow s'$
\STATE $q \leftarrow q'$
% \STATE $q \leftarrow q'$
% \STATE $j \leftarrow j+1$
\ENDFOR
% \STATE $i \leftarrow i+1$
\ENDFOR
\FOR{each state (s,q)}
\STATE $\sigma(s,q) = \max_{a \in \actions}\Qf(s,q,a)$
\ENDFOR
\end{algorithmic}
\end{algorithm}
Recall that he reward function $rew$ is defined based on $r'$, i.e,
$$
 rew((s,q),a,\tau) = \begin{cases}
 \tau & \text{if ($s,q$) is an accepting state}\\
 0 & \text{otherwise}
 \end{cases}
 $$

\end{document}